\documentclass{article}

\usepackage{microtype}
\usepackage{graphicx}
\usepackage{subfigure}
\usepackage{booktabs} 
\usepackage{hyperref}

\usepackage[accepted]{icml2024}

\usepackage{amsmath}
\usepackage{amssymb}
\usepackage{mathtools}
\usepackage{amsthm}

\usepackage[capitalize,noabbrev]{cleveref}

\theoremstyle{plain}
\newtheorem{theorem}{Theorem}[section]
\newtheorem{proposition}[theorem]{Proposition}

\newtheorem{corollary}[theorem]{Corollary}
\theoremstyle{definition}
\newtheorem{definition}[theorem]{Definition}

\theoremstyle{remark}

\usepackage[disable,textsize=tiny]{todonotes}

\usepackage[T1]{fontenc}    
\usepackage{hyperref}       
\usepackage{url}            
\usepackage{booktabs}       
\usepackage{amsfonts}       
\usepackage{nicefrac}       
\usepackage{microtype}      
\usepackage{xcolor}         
\usepackage{microtype} 
\usepackage{booktabs}  
\usepackage{url}  
\usepackage{mdframed}
\usepackage{bm}
\usepackage{amsmath}
\usepackage{amsthm}
\usepackage{xspace}
\usepackage{caption}
\usepackage{multirow}
\usepackage{booktabs}
\usepackage{natbib}
\usepackage[export]{adjustbox}
\usepackage{hyperref}
\hypersetup{
    colorlinks=true,
    linkcolor=blue,
    filecolor=magenta,      
    urlcolor=blue,
    pdftitle={Overleaf Example},
    pdfpagemode=FullScreen,
    }
\usepackage{amsmath}
\usepackage{placeins}
\usepackage{bbm}
\usepackage{wrapfig}
\usepackage{lipsum}
\usepackage[T1]{fontenc}       \usepackage{booktabs}
\usepackage[T1]{fontenc}
\usepackage{amsmath, amsfonts, amssymb}
\usepackage{ifthen}
\usepackage{booktabs}
\usepackage{microtype}
\usepackage{placeins}
\usepackage{graphicx}
\usepackage{xspace}

\usepackage[detect-weight=true, detect-family=true]{siunitx}

\usepackage[detect-weight=true, detect-family=true]{siunitx}

\newcommand{\round}[1]{\num[round-mode=places,round-precision=2]{#1}}
\newcommand{\acc}[2]{\round{#1}\ifthenelse{\equal{#2}{}}{}{\tiny ${\scriptstyle \pm}$\round{#2}}}
\newcommand{\ouralg}{\textsc{Equitable PM}\xspace}
\newcommand{\bacc}[2]{\bf \round{#1}\ifthenelse{\equal{#2}{}}{}{\bf \tiny ${\scriptstyle \pm}$\round{#2}}}

\usepackage{tikz}

\usepackage{mathtools}

\icmltitlerunning{Building Socially-Equitable Public Models}

\begin{document}

\twocolumn[
\icmltitle{Building Socially-Equitable Public Models}

\begin{icmlauthorlist}
\icmlauthor{Yejia Liu}{yyy}
\icmlauthor{Jianyi Yang}{yyy}
\icmlauthor{Pengfei Li}{yyy}
\icmlauthor{Tongxin Li}{xxx}
\icmlauthor{Shaolei Ren}{yyy}
\end{icmlauthorlist}

\icmlaffiliation{yyy}{University of California, Riverside, United States}
\icmlaffiliation{xxx}{The Chinese University of Hong Kong, Shenzhen, China}

\icmlcorrespondingauthor{Shaolei Ren}{sren@ece.ucr.edu}
\icmlcorrespondingauthor{Yejia Liu}{yliu807@ucr.edu}

\vskip 0.3in
]



\printAffiliationsAndNotice{}  

\begin{abstract}
Public models offer predictions to a variety of downstream tasks
and have played a crucial role in various AI applications, showcasing their proficiency in accurate predictions. However, the exclusive emphasis on prediction accuracy may not align with the diverse end objectives of downstream agents. Recognizing the public model's predictions as a service, we advocate for integrating the objectives of downstream agents into the optimization process. Concretely, to address performance disparities and foster fairness among heterogeneous agents in training, we propose a novel Equitable Objective.  This objective, coupled with a policy gradient algorithm, is crafted to train the public model to produce a more equitable/uniform performance distribution across downstream agents, each with their unique concerns.
Both theoretical analysis and empirical case studies have proven the effectiveness of our method in advancing performance equity across diverse downstream agents utilizing the public model for their decision-making. Codes and datasets are released at \url{https://github.com/Ren-Research/Socially-Equitable-Public-Models}.

\end{abstract}

\section{Introduction}
\begin{figure*}[t]
    \centering
\includegraphics[width=.95\textwidth]{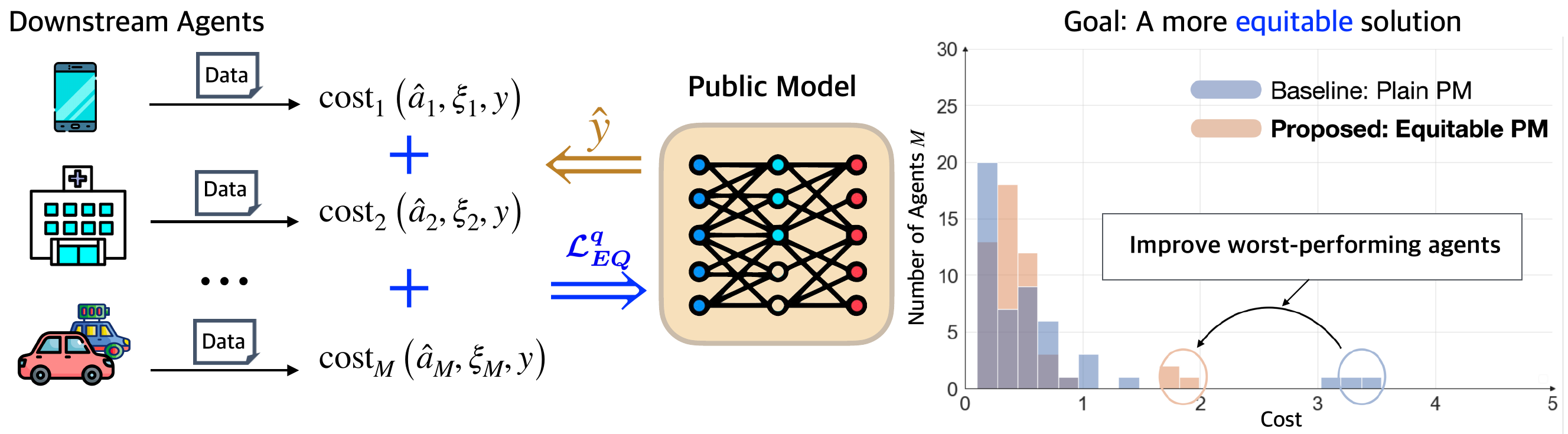}
    \caption{The \ouralg leads to a fairer solution by fostering a more equitable/uniform performance distribution across downstream agents. The embedded Equitable Objective $\mathcal{L}_{EQ}^q$ directly accounts for the decision costs across diverse downstream agents that use the prediction $\hat y$ from the public model $f$ for making informed decisions via agent-specific decision processes and actions $(\hat a_1, ..., \hat a_M)$.
    }
    \label{fig:main}
\end{figure*}

Public models whose outputs are utilized by multiple agents
have become essential building blocks for multiple AI applications such as climate modeling and traffic prediction. These models undergo training on extensive datasets and are tailored for specific domains, making them highly effective in generating accurate predictions~\cite{nguyen2023climax, bommasani2021opportunities, shah2022lmnav}. Their accessibility and availability to the public enable the widespread utilization by diverse downstream agents for individual business goals~\cite{yang2023foundation}.
However, it is important to note that exclusive reliance on prediction accuracy may not be ideal when serving a diverse range of downstream agents, each with unique business objectives. Consider a scenario where a public model predicts disease outbreaks across different regions. While accuracy is pivotal, optimizing the allocation of healthcare resources—ensuring sufficient medical supplies, personnel, and preventive measures—takes precedence, based on the general prediction provided by the public model.

We therefore suggest taking into account the impact of a public model's prediction on downstream agents, rather than solely focusing on minimizing prediction errors during training. A closely related topic is \emph{decision-focused learning}, which involves incorporating domain-specific constraints and/or objectives into the learning algorithm~\cite{game, wilder2020end}. However, the majority of existing decision-focused learning works only address a single task or agent, rendering them barely applicable to the challenge faced by public models, which deal with diverse downstream agents with their varied decision-making objectives~\cite{mandi2023decisionfocused}. Additionally, in the current decision-focused learning framework, performance disparities can arise, with certain agents consistently facing inferior outcomes. For instance, this may happen when some agents have limited training data availability, while others have access to abundant and various datasets.  

We view the prediction provided by a public model as a service for diverse downstream agents. 
As a service provider, prioritizing accuracy is crucial, but ensuring high-quality service for all users, given their diverse concerns, is even more vital. Unfairly benefiting or disadvantaging model performance on specific agents is unjust. While machine learning fairness studies primarily concentrate on achieving accuracy balance among protected groups with sensitive characteristics~\cite{barocas2023fairness, pessach2022review}, we introduce a different fairness perspective centered on ensuring performance equity/uniformity across downstream agents with different decision processes. Recent works have proposed a related concept referred to as the ``good-intent'' fairness, primarily focused on preventing overfitting to any specific device in federated learning~\cite{mohri2019agnostic, li2020fair}. However, its scope is limited to maximizing the performance of the worst-performing devices without accounting
for the decision processes and objectives of diverse downstream agents.

In this work, we propose the Equitable Objective, inspired by the $\alpha$-fairness in resource allocation~\cite{alpha-fair1}, to tackle fairness concerns while considering decision-makings of downstream agents in the development of a public model. The objective minimizes an aggregated reweighted loss, parameterized by $q$, prioritizing the optimization of worse costs—assigning higher relative weights to agents with higher downstream costs when leveraging predictions from a public model. As shown in the motivating example illustrated in Figure~\ref{fig:main}, the proposed approach leads to a more equitable performance among heterogeneous agents compared to the baseline, which solely minimizes prediction errors through MSE loss.

\textbf{Contributions.} We consider a novel setting and propose an Equitable Objective to ensure performance equity/uniformity across downstream agents leveraging a public model for decision-making. We then present an algorithm to optimize the proposed Equitable Objective, which is applicable to both differentiable and non-differentiable downstream cost functions.  Additionally, we provide theoretical results guaranteeing performance equity/uniformity of the proposed approach, along with insights into generalization bounds. Empirically, we demonstrate through case studies using real-world datasets that our approach leads to a more equitable/uniform cost distribution among downstream agents under various settings.

\section{Problem Formulation}
Consider a public model, denoted by $f:\mathsf{X}\times\Theta\rightarrow\mathsf{Y}$ where $\mathsf{X}$ is an input space, $\Theta$ is a set of parameters, and $\mathsf{Y}$ is an output space. The inputs $x\in \mathsf{X}$ are features shared by multiple downstream tasks.   For any $x\in\mathsf{X}$ and $\theta\in\Theta$,
we write $\hat y := f(x;\theta)$ as a prediction from the public model $f$. A significant emphasis in prevalent public model training is on minimizing prediction errors and achieving high accuracy~\cite{bommasani2021opportunities}. However, the loss function used for model training can be easily misaligned with the ultimate goal, which is to optimize decision-making when utilized by diverse downstream agents. We therefore suggest \emph{incorporating downstream agents' costs} into the objective formulation. 

Suppose that there are $M$ heterogeneous downstream agents employing the public model $f$ for decision-making in a stochastic environment. Each agent $m$ possesses a context variable $\xi_m$, which can either represent public shared features like local weather conditions or encapsulate unique features of downstream agents.
By following a policy $\pi_m$, each agent generates an action w.r.t. the input, denoted as $\hat {a}_m (\theta) := \pi_m(\hat y, \xi_m)$.  
The resulting action $\hat {a}_m (\theta)$ taken by the agent $m$ would incur a cost, represented as $\text{cost}_m(\hat {a}_m(\theta),\xi_m, y)$.  

To address the decision cost of downstream agents, a straightforward
approach is to formulate
the objective as 
\begin{align*}
\min_\theta \sum_{m=1}^{M} \mathbb{E} \left[\text{cost}_m (\hat a_m (\theta), \xi_m, y) - \text{cost}_m ( a_m, \xi_m, y))\right],
\end{align*}
where $a_m=\arg\min_{a\in\mathsf{A}} \text{cost}_m ( a, \xi_m, y)$ and $\mathsf{A}$ represents the action space. That is, the objective
is to minimize the total expected \emph{regret} (i.e., the cost
of decisions made based on predicted $\hat y$ minus
the cost of decisions based on the true $y$) for all the $M$ downstream
agents due to the public model's potential prediction errors.\footnote{Our study can be easily generalized to incorporate an additional weight for the expected regret of each agent.}

In an illustrative example where the public model $f$ optimizes traffic signal timings, the standard accuracy goal is to minimize delays, $\min_\theta E[(y-\hat y)^2]$, where $y$ is actual traffic conditions and $\hat y$ is predicted traffic flow. 
In reality, the transportation system involves diverse downstream stakeholders with unique concerns: commuters prioritize travel time and fuel consumption, public services focus on schedule, and environmental regulators are concerned with carbon emissions.  Each party faces decision costs from the actions it takes based on the model's predictions $\hat y$. Thus, using an objective encompassing diverse costs from downstream agents can explicitly incorporate their concerns.

Nevertheless, due to the heterogeneity of agents such as various data biases, solely minimizing the total cost objective can result in significant performance disparities among downstream agents. Consequently, certain agents may consistently experience the poorest performance when using the prediction provided by the public model compared to other agents. For example, the trained model may exhibit a preference towards the agents with greater numbers of data samples. This inequity in performance highlights concerns regarding the fairness of the services these agents receive when viewing the prediction from the public model as a shared ``resource'' serving diverse downstream agents.

\section{Fair Public Model for Downstream Agents}

To achieve fairness for the downstream agents with different decision processes, we propose  \ouralg, which seeks to optimize a novel Equitable Objective.   

\subsection{Defining Fairness: An Equitable Objective }
We now introduce the Equitable Objective to address fairness concerns in the context of diverse downstream costs across different agents. 
By drawing inspiration from the $\alpha$-fairness resource allocation~\cite{alpha-fair1, alpha-fair2, li2020fair}, we propose an objective to promote performance equity/uniformity parameterized by $q\geq0$. Both theoretical proofs (Section~\ref{s:uniTheory}) and empirical case studies (Section~\ref{s:exps}) have shown that the use of the Equitable Objective leads to a more equitable/uniform performance distribution across downstream agents. 

The Equitable Objective aims to minimize the aggregated cost incurred by downstream agents, parameterized by $q$, when utilizing the prediction from the public model $f$, as shown in Eq.~(\ref{eq:obj-nob}),
\begin{equation}
\label{eq:obj-nob}
\begin{split}
\min_\theta  \mathcal{J}_{EQ}^q (\theta)\coloneqq 
\sum_{m=1}^M \mathbb{E}^{q+1} &\left[\text{cost}_m (\hat a_m (\theta), \xi_m, y)\right. \\
&\left.- \text{cost}_m (a_m, \xi_m, y)\right], 
\end{split}
\end{equation}
where the hyperparameter $q\geq0$ promotes performance equity among different agents.
Specifically, when $q$ is set larger, the minimization process will take into account the agent of worst performance to a greater extent.

To train a public model, we need to empirically approximate \eqref{eq:obj-nob} with training data samples.  Let $\mathcal{D}_m=\{x_{m,i},y_{m,i},\xi_{m,i} | i\in [N_m]\}$ be the dataset of the agent $m$, where $N_m$ is the number of data examples in the agent $m$. Note that 
the public model's input features may still vary among different agents (e.g., a public carbon-intensity prediction model uses location-specific features to predict the local grid's carbon intensity).
Thus, for two different agents $m_1$ and $m_2$,  the public variables $\{x_{m_1,i},y_{m_1,i}\}$ and $\{x_{m_2,i},y_{m_2,i}\}$ can be identical or different depending on factors such as whether they are collected at the same time and/or location.
Given the datasets $\mathcal{D}_1,\cdots,\mathcal{D}_M$, we can approximate the expectation $\mathcal{J}_{EQ}^q(\theta)$ in Eq.~(\ref{eq:obj-nob}) with the empirical loss $\mathcal{L}_{EQ}^q(\theta)$ defined in Eq.~(\ref{eq:obj-mid-nob}),
\begin{equation}\label{eq:obj-mid-nob}
\begin{split}
\min_\theta \mathcal{L}_{EQ}^q (\theta)\!\!:=& \sum_{m=1}^M \biggl [\Bigr(\frac{1}{N_m} \sum_{i=1}^{N_m} C_{m, i}\Bigr)^{q+1}\biggl],
\end{split}
\end{equation}
where we denote $C_{m, i} = \text{cost}_m (\hat a_{m,i}(\theta), \xi_{m, i}, y_{m,i}) -\text{cost}_m ( a_{m,i}, \xi_{m, i}, y_{m,i})$ 
as the regret regarding the $i$th  sample of agent $m$.

In practice, the public model developer may not always
have direct access to the costs of all the downstream agents for training. In
such cases, it can generate synthetic downstream agents by
modeling their decision processes based
on, e.g., utility maximization or cost minimization, for the target application. Additionally, annotation-sample efficient methods
like task programming~\cite{suntaskp} can also help model the downstream
decision processes.

It is worth noting that we have also proposed a more general objective in Appendix~\ref{app:withBeta}, which combines $\mathcal{L}_{EQ}^q(\theta)$ with the public model's prediction loss $\mathcal{L}_f$ via a balancing hyperparameter $\beta\in[0,1]$, allowing for a more nuanced control over the fairness-accuracy trade-off in optimization. In the subsequent text, we denote $\mathcal{L}_{EQ}^q (\theta)$ as $\mathcal{L}^q(\theta)$ for simplicity. 

Our proposed $\mathcal{L}^q(\theta)$ ensures that the public model's predictions consider the diverse concerns of downstream agents. 
The trade-offs introduced by adjusting  $q$ contribute to a fairer distribution of performance across agents, fostering an equitable decision-making environment. In the subsequent sections, we provide details and algorithms to train a public model with the Equitable Objective.

\subsection{Training Public Model: \ouralg}
\label{s:Solver}
The difficulties in training a public model vary depending on the cost functions. When the cost functions are differentiable, it is feasible to calculate the gradient based on the chain rule. 
By back propagation, we can get the gradient as
\begin{align*}
\begin{split}
 \nabla_\theta \mathcal{L}^q(\theta) \!=  \!\!\sum_{m=1}^M \sum_{i=1}^{N_m} \nabla_{C_{m, i}} \mathcal{L}^q \!\nabla_{\hat a_{m, i}} \!C_{m, i} \nabla_{\hat y_{m,i}} \hat a_{m, i} \nabla_{\theta} \hat{y}_{m,i},
\end{split}
\end{align*}
where we denote the regret of the $i$th sample of agent $m$ as $C_{m, i} =  \text{cost}_m (\hat a_{m,i}, \xi_{m, i}, y_{m,i}) -\text{cost}_m ( a_{m,i}, \xi_{m, i}, y_{m,i}).$

\begin{algorithm}[t!]
   \caption{\ouralg}
   \label{alg:solver}
\begin{algorithmic}
   \STATE {\bfseries Input:} Training dataset, learning rate $\alpha$
   \STATE Initialize the parameters $\theta$
   \FOR{each batch $k\in [K]$}
    \STATE Obtain $\hat y_{m,k,i}$ by the public model $f(\cdot;\theta)$
   \STATE Compute the cost regret $C_{m,k,i}$ for the example ($x_{m,k,i}$, $y_{m, k,i}$, $\xi_{m, k,i}$) in batch $k$ for $m \in [1, ..., M] $
    \STATE Compute the gradient $\nabla_\theta \mathcal{L}_k^q(\theta)$ for batch $k$ by Eq.~(\ref{eq:grad-nob-empirical}). 
   \STATE Update the parameter $\theta \leftarrow \theta - \alpha \nabla_\theta \mathcal{L}_k^q(\theta)$
   \ENDFOR
\end{algorithmic}
\end{algorithm}

Nonetheless, the training becomes significantly more challenging when the cost function is non-differentiable w.r.t. the $\hat y$. The non-differentiable cost function is prevalent for many practical downstream tasks. For example, some downstream tasks are combinatorial optimization problems with discrete actions \cite{wilder2019melding}. Thus, a training method that does not rely on differentiable cost functions is critically needed for public models.

One possible method is to learn a differentiable model to approximate the cost function by observing the evolution of the sequence of actions and costs~\cite{moerland2023model, yu2020mopo}. However, this method suffers from potentially inaccurate modeling of dynamic environments ~\cite{agarwal2023vo, pmlr-v89-malik19a}. Therefore, we can adopt a model-free approach, such as black-box optimization, which requires fewer assumptions about the underlying system~\cite{agarwal2023vo}. In our context, we choose the policy gradient (PG) algorithm, falling into the category of model-free approaches, in favor of its natural exploration-exploitation trade-off~\cite{bhandari2024global, peters2006policy}. We next present the process of using PG to optimize $\mathcal{L}^q (\theta)$.

PG for training a public model differs notably from the standard PG algorithm. Due to the non-separable Equitable Objective in Eq.~\eqref{eq:obj-mid-nob}, the supervision loss hinges on the average regret of each agent $m$. Thus, we use a batch-based training approach. At each training step, we employ a probabilistic public model  $\sigma_{\theta}(\hat{y}\mid x)$ to sample $\hat{y}_{m,i}$ given inputs $x_{m,i}, i\in [1,\cdots, B_m]$ from a batch of $B_m$ samples. By this way, we obtain an equitable loss expressed as $\sum_{m=1}^{M}  \Bigr(\frac{1}{B_m} \sum_{i=1}^{B_m} C_{m, i}\Bigr)^{q+1}$. To utilize the equitable batch loss for supervising the training of the public model, we reformulate the original objective as 
\[
\begin{split}
     \mathbb{E}[\mathcal{L}^q(\theta)] = \mathbb{E}_{(X,Y,\hat{Y},\Xi) \sim p_\theta } \biggl[\sum_{m=1}^{M}  \Bigr(\frac{1}{B_m} \sum_{i=1}^{B_m}  C_{m, i}\Bigr)^{q+1}  \biggl], 
 \end{split}
 \]
where $p_\theta$ is the joint distribution of the random variables $X=[x_{m,i}\mid m\in[1,\cdots,M], i\in [1,\cdots, B_m]]$, $Y=[y_{m,i}\mid m\in[1,\cdots,M], i\in [1,\cdots, B_m]]$, $\hat{Y}=[\hat{y}_{m,i}\mid m\in[1,\cdots,M], i\in [1,\cdots, B_m]]$, and $\Xi=[\xi_{m,i}\mid m\in[1,\cdots,M], i\in [1,\cdots, B_m]]$, which relies on the probabilistic public model  $\sigma_{\theta}(\hat{y}\mid x)$.
The gradient of $\mathbb{E}[\mathcal{L}^q(\theta)]$ with respect to $\theta$ is given by
\begin{align}
\begin{split}
     \nabla_\theta \mathbb{E}[\mathcal{L}^q(\theta)] \!= \!  &\mathbb{E}_{(X,Y, \hat{Y}, \Xi) \sim p_\theta } \!\Bigr\{\biggl[\sum_{m=1}^{M} \!\! \;\sum_{i=1}^{B_m}\!\!\nabla_\theta\!\log \sigma_\theta (\hat y_{m,i}| x_{m,i})\biggl] \\ &
    \cdot \biggl[   \sum_{m=1}^{M}  \Bigr(\frac{1}{B_m} \sum_{i=1}^{B_m} C_{m, i}\Bigr)^{q+1}
  \biggl] \Bigr\}, \label{eq:grad-nob}
\end{split}
\end{align}
whose detailed derivation 
can be found in Appendix~\ref{s:app-grad}. 

Given a training dataset with $K$ batches, we can get an empirical approximation of the expected gradient in Eq.~\eqref{eq:grad-nob} as follows
\begin{align}
\begin{split}
     \nabla_\theta \mathcal{L}^q(\theta) \!= \!  &\frac{1}{K}\sum_{k=1}^K \nabla_\theta \mathcal{L}_k^q(\theta),\label{eq:grad-nob-empirical}
\end{split}
\end{align}
where $\nabla_\theta \mathcal{L}_k^q(\theta)=\biggl[\sum_{m=1}^{M} \!\!\; \sum_{i=1}^{B_m} \!\!\nabla_\theta\!\log \sigma_\theta (\hat y_{m,k,i}| x_{m,k,i})\biggl]
   \\\cdot\biggl[   \sum_{m=1}^{M}  \Bigr(\frac{1}{B_m} \sum_{i=1}^{B_m} C_{m, k,i}\Bigr)^{q+1}
  \biggl]$.

In summary, the training steps using PG to minimize $\mathcal{L}_\theta^q$ in can be outlined as in the Algorithm~\ref{alg:solver}. During inference, the public model is its deterministic counterpart expressed as $f(x; \theta) := \arg \max_{\hat{y}} \sigma_\theta (\hat{y}|x)$. In subsequent texts, we refer to our proposed method as the \ouralg.

\vspace{-3mm}
\section{Theoretical Analysis}
\label{s:uniTheory}
\subsection{Performance Equity/Uniformity}
In this section, we provide the theoretical justification that the proposed Equitable Objective can promote greater equity/uniformity in the performance distribution across downstream tasks with proofs in Appendix~\ref{app:proofs}. We use $C_m=\frac{1}{N_m}\sum_{i=1}^{N_m}C_{m,i}$ to denote the performance of the $m$-th downstream agent. 
We here adopt \textit{variance} and \textit{entropy} to measure the uniformity of the performance distribution across downstream tasks.
\begin{definition} 
\textit{(Equity by Variance)}
\label{def1}
The performance distribution of $M$ downstream agents $\{C_1(\theta), ..., C_M(\theta)\}$ is more equitable/uniform under solution $\theta$ than $\theta'$ if 
\begin{equation}
 {\mathsf{Var}}\bigr(C_1(\theta), ..., C_M(\theta)\bigr) < \mathsf{Var}\bigr(C_1(\theta'), ..., C_M(\theta')\bigr),
\end{equation}
\end{definition}
where $\mathsf{Var}$ represents the \textit{variance} of performance.

\begin{definition}
\textit{(Equity by Entropy)}
\label{def2}
The performance distribution of $M$ downstream agents $\{C_1(\theta), ..., C_M(\theta)\}$ is more equitable/uniform under solution $\theta$ than $\theta'$ if the entropy of the normalized performance distribution satisfies
\begin{equation}
\mathbb{H}_\text{norm}\bigr(C(\theta)\bigr) \geq \mathbb{H}_\text{norm}\bigr(C(\theta')\bigr),
\end{equation}
where $\mathbb{H}_\text{norm}\bigr(C(\theta)\bigr)$ is expressed as
\begin{equation}
\hspace{-8pt}
    -\sum_{m=1}^{M} \frac{C_m(\theta)}{\sum_{m=1}^M C_m{(\theta)}} \log \left(\frac{C_m(\theta)}{\sum_{m=1}^M C_m{(\theta)}} \right).
\end{equation}
\end{definition}

Definition \ref{def1} and \ref{def2} are also considered in in \cite{li2020fair} and offer metric definitions for evaluating performance equity among agents. Specifically, higher variance or a lower $\mathbb{H}_\text{norm}\bigr(C(\theta)\bigr)$ indicates larger variability (\textit{i.e.,} less equity) in the performance across agents.

We next provide theorems showing that the Equitable Objective Eq.~\eqref{eq:obj-nob} can encourage a more fair solution according to Definition~\ref{def1} and~\ref{def2}. We initiate the analysis with the special case of $q=1$, and prove that $q=1$ can lead to a more equitable performance distribution than $q=0$. The notation $\theta^*_q$ denotes the global optimal solution of $\min_\theta \mathcal{L}^q(\theta)$.
\begin{theorem}
\label{proposition:beta-0-q-1}
When $q=1$, the optimum of Equitable Objective is more equitable compared to $q=0$, indicated by smaller variance of the model performance distribution, \textit{i.e.}\, $\mathsf{Var}(C_1(\theta^*_{q=1}), ..., C_M(\theta^*_{q=1})) < \mathsf{Var}(C_1(\theta^*_{q=0}), ..., C_M(\theta^*_{q=0}))$.
\end{theorem}

Moving forward to the general case, we show that for any $q > 0$, the proposed Equitable Objective can achieve better uniformity in performance distribution given a small increase of $q$.

\begin{theorem}
\label{proposition:2}
Let $C(\theta)$ be twice differentiable in $\theta$ with $\nabla^2 C(\theta) > 0$ (positive definite), for any $M \in \mathbb{N}$, the derivative of $\mathbb{H}_\text{norm}\bigr(C^{q+1}(\theta^*_p)\bigr)$ w.r.t. the evaluation point $p$ is non-negative, \textit{i.e.},
\begin{equation}
\label{eq:p=q}
    \frac{\partial \mathbb{H}_\text{norm}\bigr(C^{q+1}(\theta^*_p)\bigr)}{\partial p}|_{p=q} \geq 0.
\end{equation}
\end{theorem}

Theorem~\ref{proposition:2} establishes that a positive partial derivative of $\mathbb{H}_\text{norm}\bigr(C^{q+1}(\theta^*_p)\bigr)$ signifies that a small increase in $p$ is associated with a greater degree of performance uniformity in the learning outcome~\cite{Beirami2019}. 

\subsection{Generalization Bounds}
\label{sGB}
Denote $h$ as the hypothesis function of the public model, i.e. $h(x)=f(x,\theta)$. 
In this work, we prove that the proposed Equitable Objective in Eq. (\ref{eq:obj-mid-nob}) enables the public model to generalize well on the equitable loss described in Eq.~\eqref{eq:equi-loss-bounds}~\cite{mohri2019agnostic}.
\begin{equation}
\label{eq:equi-loss-bounds}
    \mathcal{J}_{\kappa} (h) = \sum_{m=1}^{M} \kappa_m \mathbb{E}_{(x, y) \sim D_m} C_m(h(x), y),
\end{equation}
where $\kappa=[\kappa_1,\cdots, \kappa_M]$ lies in a probability simplex $\Delta$. 

To show the generalization bound, we first give an equivalence of the Equitable Objective in Eq. (\ref{eq:obj-mid-nob}). Given the definition of dual norm, we have
\begin{equation}
\begin{split}
   & \Tilde{\mathcal{L}}^q(h) = (\mathcal{L}^q(h))^{\frac{1}{q+1}}\\
    =&\max_{v, ||v||_p\leq 1} \sum_{m=1}^{M} \Bigr(\frac{v_j}{N_m} \sum_{i=1}^{N_m} C_m(h(x_{m,i}), y_{m,i})\Bigr),
    \label{eq:obj-kappa}
    \end{split}
\end{equation}
where $\frac{1}{p} + \frac{1}{q+1}=1$ $(p \geq 1, q \geq 0)$. Thus, the proposed Equitable Objective in Eq. (\ref{eq:obj-mid-nob}) is equivalent to minimizing the empirical loss $\Tilde{\mathcal{L}}^q(h)$ in Eq. \eqref{eq:obj-kappa}.
We present the generalization bound for $\mathcal{J}_{\kappa}(h)$ which depends on $\Tilde{\mathcal{L}}^q(h)$ as below. 

\begin{proposition}
\label{proposition:5}
Assume that the cost functions $cost_m$ are bounded by $B$. Then for any $\delta > 0$, with probability at least $1-\delta$, the following holds for any $\kappa$ in a probability simplex $\Delta$, and any $h \in H$:
\begin{align}\label{eq:gen2}
\begin{split}
    \mathcal{J}_\kappa(h)  \leq &\max_{ \kappa \in \Delta}(||\kappa||_p)\Tilde{\mathcal{L}}^q(h)
 + \max_{ \kappa \in \Delta} \Bigr(\mathbb{E}[\max_{h \in H} \mathcal{J}_\kappa(h) \\ & -  \mathcal{L}_\kappa(h)]\Bigr)
    + B(\sqrt{\sum_{m} \frac{\kappa_m^2}{2N_m}\log\frac{1}{\delta}})\Bigr),
\end{split}
\end{align}
where $\frac{1}{p} + \frac{1}{q+1}=1$, $\Tilde{\mathcal{L}}^q(h)$ is the equivalent Equitable Objective in Eq.~\eqref{eq:obj-kappa}, and
$
    \mathcal{L}_{\kappa} (h) = \sum_{m=1}^{M} \frac{\kappa_m}{N_m} \sum_{i=1}^{N_m} C_m(h(x_{m, i}), y_{m, i})
$ is the empirical loss of $\mathcal{J}_\kappa(h)$.
\end{proposition}

\section{Empirical Case Studies}
\label{s:exps}
We evaluate the effectiveness of \ouralg in fostering a more equitable solution for downstream heterogeneous agents, each with their own business objective, while utilizing the prediction from an upstream public model. Our empirical study encompasses the applications of data centers and Electric Vehicles (EV) charging.

\paragraph{Evaluation Metrics} Instead of solely prioritizing the prediction accuracy, we emphasize the outcome, \textit{e.g.}, decision cost (or rewards), of downstream agents from using the prediction of the upstream public model. Moreover, for diverse agents, we believe the algorithm should promote an equitable/uniform distribution of performance rather than disproportionately affecting specific agents. Our evaluation therefore incorporates three key metrics to assess the uniformity of performance distribution across agents:
1) \textit{Variance} of the cost regret; 2) \textit{Mean} of the cost regret; and 3) \textit{$ {C_{95}-C_{5}}$ percentile}, the discrepancy between the $95\%$ and $5\%$ percentiles of the cost regret across agents.

\subsection{Application I: Carbon Efficiency in Data Centers}
\label{sAPPI}

\textbf{Setup}\quad
Data centers are responsible for a significant amount of energy consumption and carbon emissions. In order to reduce their carbon footprint, it is crucial to manage energy consumption and optimize the allocation of workloads~\cite{radovanovic2022carbon,patterson2022carbon}.

In empirical studies, we denote the workload demand of data center $j$ at time step $t$ as $w_{j, t}$, represent the allocated computational resource as $p_{j,t}$, and indicate the predicted carbon emission rate at time $t$ by $c_t$, where $c_t$ is estimated by a public model. The processing delay can be calculated as $\frac{w_{j,t}}{p_{j,t} - w_{j,t}}$. Our objective is to minimize the combined impact of carbon emissions, $p_{j,t} c_t$, and processing latency, 
$\frac{w_{j, t}}{p_{j,t} - w_{j,t}}$, by determining the optimal allocation of computational resource $p_t$, as shown in Eq. (\ref{eq:carbon}),
\begin{align}
\begin{split}
   \min_{p_{j,t}}  \quad
   p_{j,t} c_{t} + \lambda_j \frac{w_{j,t}}{p_{j,t} - w_{j,t}},
\label{eq:carbon}
\end{split}
\end{align}

where $\lambda_j$ adjusts the relative significance of carbon emissions and processing latency for different data centers.

\paragraph{Datasets}
Our experiments mainly use the publicly available state-level energy fuel mix dataset~\cite{EIA} and the Azure cloud workload dataset~\cite{shahrad2020serverless}. The fuel mix dataset provides information on various energy sources utilized in electricity generation (e.g. coal, natural gas, and oil) while the Azure cloud workload dataset captures the energy consumption/demand patterns of the cloud center across different time periods.
Besides, we utilize the carbon conversion rates provided in~\cite{green} to calculate the carbon emissions associated with different types of fuel used for energy generation. More details are in Appendix~\ref{s:app-data}.

\paragraph{Implementation Details} We set the number of downstream agents as $50$. We set up 3 different settings by varying data distribution and the values of $\lambda$ among agents. The $50$ agents have Wasserstein distance of $w_j$ ranges falling within $[0.03, 0.58]$ and they are labeled as ``similar agents''. At the same time, we randomly select $20$ agents from the set and introduce random noise, resulting the total $50$ agents with Wasserstein distance w.r.t. $w_j$ spanning $[0.04, 57.97]$, which are labeled as ``different agents''. Additionally, regarding the values of $\lambda$, ``same $\lambda$'' designates $\lambda=2$, whereas ``different $\lambda$'' spans $\lambda=\{2, 4, ..., 100\}$ among agents. Given the time-series nature of the datasets, we train and employ an LSTM network as the shared public model. More details are provided in Appendix~\ref{app:impl-d}.

\begin{table*}[t]
    \small
    \setlength\intextsep{10pt}
    \setlength\lineskip{10pt}
    \captionsetup{width=\linewidth}
     \caption{Statistics of the test results under different setups. As $q$ increases, the variance and $C_{95} - C_{5}$ percentile of cost regret distribution across agents decrease, suggesting a more uniform distribution of costs across groups. The \ouralg also achieves lower means in costs regrets across agents compared to the Plain PM in general.
     }
    \centering
    \resizebox{0.9\textwidth}{!}{
    \setlength{\tabcolsep}{16pt}
    \begin{tabular}{cclcccc}
    \toprule
 
  Setting & Method &  $q+1$ &  \textbf{Variance} & \textbf{Mean} & $\bm{C_{95} - C_{5}}$ & MSE \\
    \toprule
  
    \multirow{4}{*}{\begin{tabular}{@{}c@{}}Similar Agents, \\ Different $\lambda$\end{tabular}} &  
     \multirow{3}{*}{\textbf{\ouralg}} & 
     $1$ &  0.0029  & 0.1591 & 0.1687 & 4.6308\\
    &  & $1.1$ & 0.0003 & 0.0544 & 0.0576 & 4.2465\\
    &  & $1.5$ & \textbf{0.0002}  & \textbf{0.0465} & \textbf{0.0493} & 4.2194\\
    \cmidrule(lr{1em}){2-7}
    & Plain PM & -  & 0.0085 & 0.2732 & 0.2897 & \textbf{4.2054}\\
    \midrule
  
    \multirow{4}{*}{\begin{tabular}{@{}c@{}}Different Agents, \\ Same $\lambda$\end{tabular}} & 
     \multirow{3}{*}{\textbf{\ouralg}} & 
    $1$&  0.0008 & 0.0909 & 0.0809 & 5.0991\\
    &  & $3$ & 0.0001 & 0.0345 & 0.0306 & 4.4338\\
    &  & $20$ &   \textbf{1.71e-5} & \textbf{0.0136} & \textbf{0.0121} & 4.2028\\
    \cmidrule(lr{1em}){2-7}
    & Plain PM & -  &  0.0009 & 0.0988 & 0.0879 & \textbf{4.2013}\\
    \midrule
   
    \multirow{4}{*}{\begin{tabular}{@{}c@{}}Different Agents, \\ Different $\lambda$\end{tabular}} &  
     \multirow{3}{*}{\textbf{\ouralg}} & 
    $1$&   0.0181 & 0.2619 & 0.4229 & 4.6607 \\
    &  & $3$ &  0.0068 & 0.1603 & 0.2588 & 4.4182\\
    &  & $10$ &  \textbf{0.0055} & \textbf{0.1444} & \textbf{0.2331} & 4.3819\\
    \cmidrule(lr{1em}){2-7}
    & Plain PM & -  & 0.0602 & 0.4779 & 0.7717 & \textbf{4.2013}\\
     \bottomrule
    \end{tabular}
    }
    \label{tab:main-res}

\end{table*}
\paragraph{Results} In Table~\ref{tab:main-res}, we present the performance comparison between \ouralg and the traditional public model that does not consider the decision-making process of downstream agents, referred to as the Plain PM. From the Table~\ref{tab:main-res}, we can observe that the values of cost regret variance and $C_{95}-C_{5}$ achieved by \ouralg are smaller than the Plain PM. The variance and percentile measure values of \ouralg decrease as the value of $q$ increases, suggesting more uniform cost regret distributions, and therefore a fairer solution.
Also, the \ouralg has resulted in an improved cost regret mean compared to the Plain PM.
Although the \ouralg does not achieve the minimum MSE on predicting carbon emissions $c_t$, it delivers more equitable and accurate cost outcomes for heterogeneous downstream agents under various settings.

Figure~\ref{fig:app1-qs} shows the distribution of cost regret 
with various $q$ values under different setups. When the data distribution among agents remains similar but with varying $\lambda$ values, an increase in the value of $q$ leads to a distribution with lower variance, as observed in Figure~\ref{fig:app1-qs} (a). In Figure~\ref{fig:app1-qs} (b), (c) and (d), we observe that the cost regret distributions become less dispersed when $q$ increases, indicating a more equitable solution for different agents.

\begin{figure*}[h]
    \centering
\includegraphics[width=\textwidth]{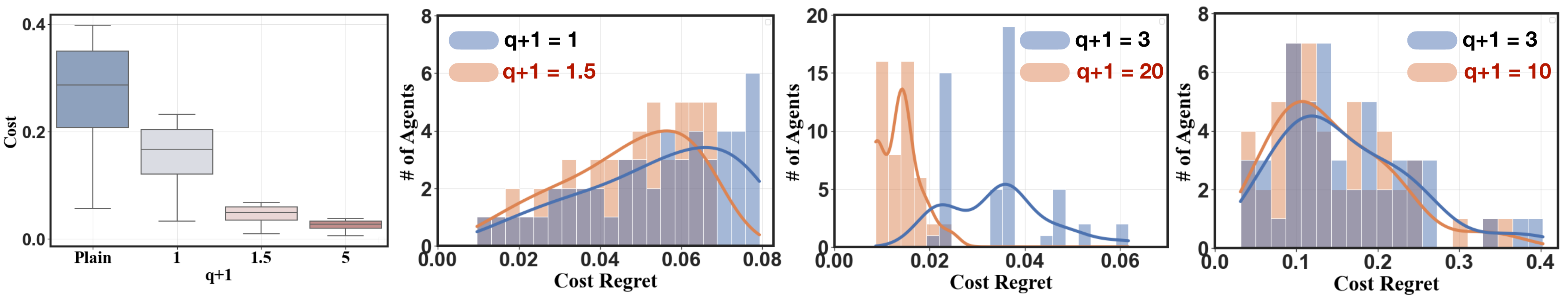}
    \caption{(a) Comparison of cost regret distributions between Plain PM vs. \ouralg on ``similar agents'' with different $\lambda$. The \ouralg shows lower variability in cost distribution compared to the Plain PM. With varied $q+1$, we show cost regret distributions when using \ouralg in (b) ``similar agents'' with different $\lambda$; (c)  ``different agents'' with same $\lambda$; (d) ``different agents'' with different $\lambda$. As the value of q increases,
the cost regret distribution across downstream agents achieves greater uniformity, implying a more equitable solution.}
    \label{fig:app1-qs}
    \vspace{-2mm}
\end{figure*}
\vspace{2mm}
\subsection{Application II: Scheduled EV Charging for Environmental Sustainability}
\label{sAPPII}
\paragraph{Setup} The increasing popularity of EV raises concerns about their environmental impact. To address this, scheduling EV charging can play a pivotal role in enhancing both environmental sustainability and the stability of the power system~\cite{filote2020environmental}. Here, we evaluate the potential of \ouralg for a more equitable solution, in the context of optimizing the EV charging schedule aiming at minimizing the financial cost, together with carbon emission and water consumption.

Consider an EV $j$ with an initial electrical charge state, denoted as $I_j$, which requires attaining an electric charge level represented as $D_j$. This charging process occurs within a defined time window that begins at ${s}_j$ and concludes at ${e}_j$. For optimization purpose, we discretize the time window $[s_j, e_j]$ into time slots $\tau=\{1, ..., T\}$ and utilize a binary charging schedule defined as $X_j$. In the schedule, each element $x_{j,t}$ is either $1$, indicating that we charge the vehicle at the time $t$, or $0$ if we don't (e.g., $X_j=[1, 0, ... 1]$, with a total of $|\tau|$ elements in $X_j$). And the amount of electricity charged at each time step $t$ of the $j$-th EV is $\zeta_{j, t}$. 

At each time step $t$ within $[s_j, e_j]$, an upstream public model predicts the combined carbon, water efficiency and electricity price, expressed as $E_{t} =  E^C_{t} +\gamma E^W_{t} + \eta E^P_{t}$, where $E^C_{t}$ and $E^W_{t}$ denote the carbon and water efficiency at time $t$, respectively, while $E^P_{t}$ represents the electricity price at time $t$ for downstream agent EV. The term $\gamma$ and $\eta$ represents their relative weight of these factors. Here, the efficiency refers to the amount of carbon emission or water consumption per unit of electricity generated.

The objective is to reduce the total cost, which includes carbon emissions, water consumption, and the financial cost of electricity incurred throughout the charging process, by determining the optimal charging schedule for the $j$-th EV. We formulate the objective in Eq.~(\ref{eq:cwp}).
\begin{align}
    \begin{split}
        \min_{X_{j}}
        &   \sum_t \zeta_{j, t} x_{j,t} \cdot E_t
        \\
        s.t. 
        \quad & 
       I_j + \sum_t \zeta_{j,t} x_{j,t} = D_j \\
        & E_{t} =  E^C_{t} +\gamma E^W_{t} + \eta E^P_{t},
    \label{eq:cwp}
    \end{split}
\end{align}
where $X_j=[x_{j,1},\cdots, x_{j,T}]$.

\textbf{Datasets} \quad Our main sources of datasets include the publicly available ACN-Data, collected from the Caltech ACN and similar websites~\cite{acndata}, as well as the California Electricity Market (CAISO)~\cite{CAISO}. The ACN-Data records the real time charging details, including EV arrival/departure times and actual energy delivered in each charging session.
Simultaneously, the CAISO provides data on electricity prices in California. We use the ACN-Data to estimate power demand and charging rates for EV in residential areas, considering that EV models are similar between residential and other charging stations~\cite{sche}. Additionally, we use the state-level energy fuel mix data~\cite{EIA} for carbon and water efficiency calculation. Regarding the available charging time window, from $s_j$ to $e_j$, in residential sectors, we use the data from The National Household Travel Survey (NHTS) to approximate~\cite{NHTS, sche}. Further details can be founded in Appendix~\ref{s:app-data}.

\begin{figure}[t]
    \centering
    \includegraphics[width=\columnwidth]{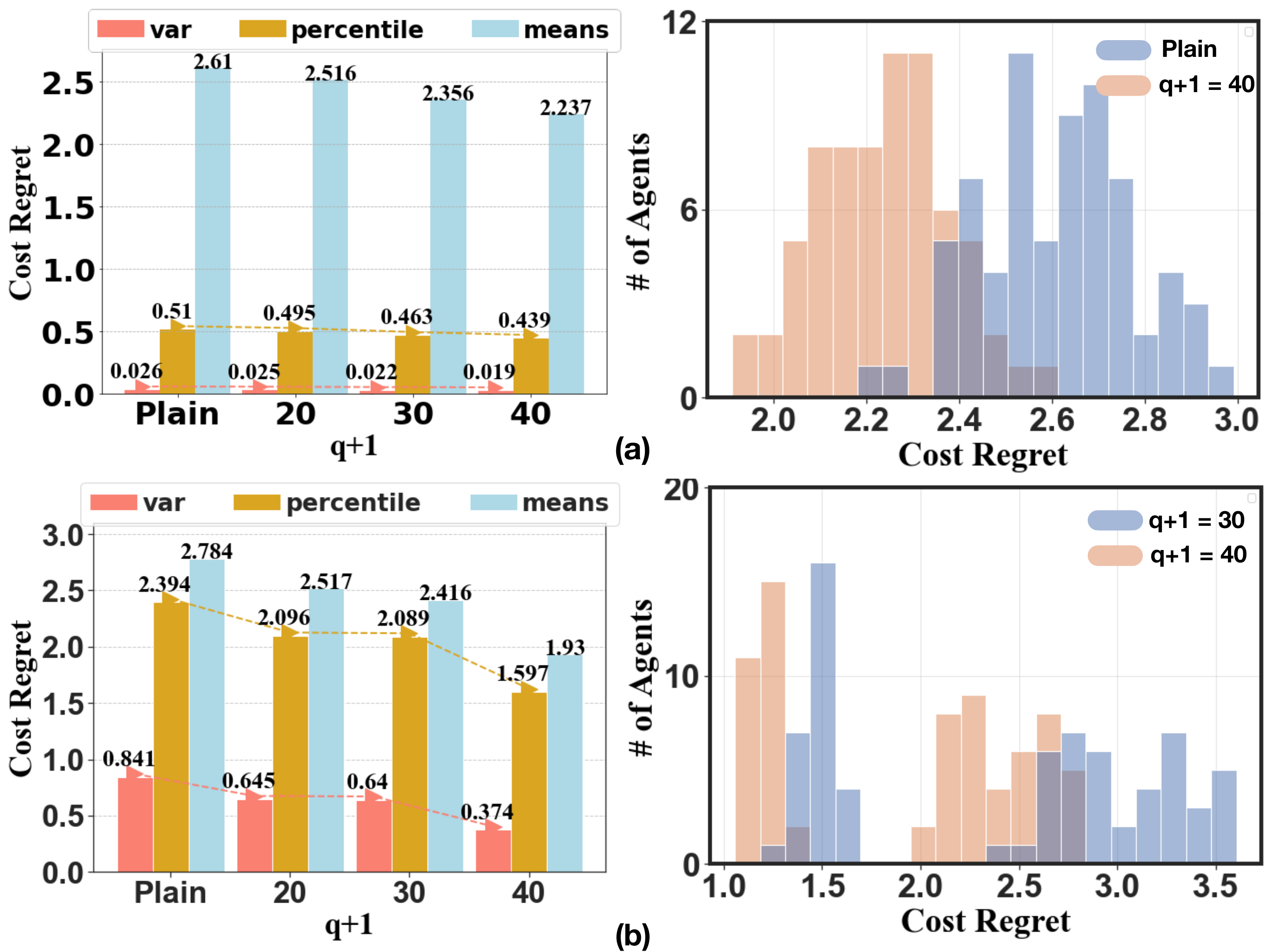}
    \caption{Statistics and cost regret distributions of test result between Plain PM and \ouralg with varied $q+1$ for (a) ``similar''; and (b) ``different'' agents.  The \ouralg demonstrates improved uniformity in agent distributions compared to the Plain PM.
    As $q$ increases, the uniformity of cost distribution across agents improves. }
    \vspace{-10pt}
    \label{fig:app2-nob}
\end{figure}

\textbf{Implementation Details} \quad In the experiments, we follow the calculation of carbon and water efficiency outlined in~\cite{li2023making}. We recognize that different EV exhibit distinct charging patterns~\cite{ev-charging-curve}. Given our central objective of ensuring fairness across a diverse range of EV, we here however opt for a simplifying assumption of a uniform charging rate, implying that $\zeta_{j, t}$ remain constant w.r.t. $t$ for the $j$-th EV~\cite{ev-charging-curve}. More specifically, this rate is calculated by the charged electricity divided by the difference between the ending charging and starting times of the $j$-th EV, as provided in the ACN-Data. Additionally, we use the energy demands of each EV provided in the ACN-Data as $D_j$. To ensure flexibility in charging scheduling, we set the time frame $|\tau|=12$. For instance, if we use an hourly unit, this corresponds to scheduling charging for half of the day. 
We set the number of downstream EV agents as $70$ and $\gamma$ and $\eta$ are set at $1$. In the experiments, we explore the effectiveness of \ouralg across agents exhibiting varied data distributions.  The Wasserstein distance range of $(D_j - I_j)$ for agents labeled as ``similar agents'' spans $[0.80, 9.00]$ , while for ``different agents'', it ranges $[1.33, 60.36]$. The \emph{Transformer} architecture, with a linear layer as the task head, is employed as the shared public model to predict $E_t$ for scheduling downstream EV charging. More details are provided in Appendix~\ref{app:impl-d}.

\textbf{Results}\quad  Figure~\ref{fig:app2-nob} reports the evaluation results between the Plain PM and \ouralg with different $q$ for agents exhibiting both similar and different distributions. The results demonstrate that the \ouralg consistently achieves lower variance and $C_{95}-C_{5}$ percentile values compared to utilizing the Plain PM in both settings.  Examining Figure~\ref{fig:app2-nob}, it becomes evident that as the value of $q$ increases, both the variance and the range $C_{95}-C_{5}$ percentile of cost regret distributions among agents decrease, indicating a trend towards a more uniformly distributed performance.

\section{Related Works}
\textbf{Fairness in Machine Learning} \quad
Fairness is a prevalent topic within the realm of machine learning, often focusing on the protection of certain groups or attributes. The problem stems partly from inherent biases within datasets and could be further magnified by models~\cite{fairnSur1, li2023fairer}.
Various approaches have been developed to mitigate this form of unfairness, spanning different stages of model development. These approaches encompass pre-processing methods, such as excluding sensitive attributes from the datasets to prevent model reliance on these factors~\cite{Biswas2021, madras2018learning}. Post-processing techniques calibrate prediction outcomes after training~\cite{pessach2022review, noriegacampero2018active}, and in-processing methodologies directly integrates fairness considerations during model training~\cite{fairnSur1, kearns2018preventing}.

Our work enforces fairness during training but takes a distinct perspective. We emphasize the equity/uniformity of performance distribution across heterogeneous agents, as we view the upstream public model as a shared resource serving diverse downstream agents. While certain studies advocate for equivalent error rates as a fairness criterion~\cite{48135}, our goal does not prioritize optimizing equal model accuracy across all agents. Drawing an analogy between the shared public model and a resource, we are inspired by a unified resource allocation framework called $\alpha$-fairness, where the service provider can adjust fairness emphasis via a single hyperparameter~\cite{879343, lan2009axiomatic}. However, the aspect of equity, specifically concerning the impact of predictions from a shared public model on diverse downstream agents' business decisions, a focal point in our work, remains unexplored in previous literatures.

\textbf{Decision-focused Learning} \quad
Decision-focused learning is an emerging area in machine learning that trains a model to optimize decisions by integrating prediction and optimization within an end-to-end system~\cite{mandi2023decisionfocused}. It diverges from the predict-then-optimize framework~\cite{balghiti2022generalization, elmachtoub2020smart}, where a ML model is trained initially to map observed features to relevant parameters of a combinatorial optimization problem, followed by using a specialized optimization algorithm to solve the decision problem based on predicted parameters. The predict-then-optimize methodology assumes  accurate predictions generate precise models, enabling optimal decisions. However, ML models often lack perfect accuracy, prediction errors thus can lead to suboptimal decisions.

In comparison, decision-focused learning directly trains the ML model to make predictions that lead to good decisions, where the optimization is embedded as a component of the ML model, creating an end-to-end approach. Recent studies have utilized supervised or reinforcement learning to optimize ultimate decisions with end-to-end machine learning~\cite{wilder2020end, game, bello2017neural, donti2019taskbased}. This holistic approach has enhanced the model's capability to drive informed and effective downstream decisions. However, few existing works have considered the issue of performance disparity across diverse business agents, each with their distinct concerns, specifically in the context of using a publicly shared model to optimize their decisions~\cite{yang2023foundation, learntodefer, learnh}.

\section{Conclusion}
In this paper, we introduce the novel Equitable Objective and its corresponding solver, the \ouralg with either differentiable or non-differentiable cost functions, to promote the performance equity/uniformity among diverse downstream agents that depend on the predictions of a shared public model for their decision-making. Alongside theoretical proofs demonstrating the performance uniformity improvement achieved by our proposed approach, the empirical case studies using real-world datasets further validates that \ouralg can attain a more equitable solution compared to methods that solely focuses on minimizing the prediction error without considering the objectives of downstream agents in different settings. 

\textbf{Limitation \& Future Works} \quad Our current method relies on accessing the decision costs from downstream groups to construct a socially-responsible public model, potentially raising privacy and security concerns. In future research, we aim to investigate ways that uphold privacy and increase robustness against adversarial attacks (e.g., maliciously reporting decision costs) when extending our approach. 
Furthermore, while the models used in the current case studies are appropriate for the present context, their scale is relatively modest, also due to a constraint imposed by our limited computing resources. We would like to explore the efficacy of our proposed method in more extensive architectures and other domains such as healthcare.
Additionally, the exploration of alternative methods, such as using fine-tuning to align public foundation models for making business-informed decisions and addressing fairness concerns accordingly, continues to be a key focus for upcoming research endeavors.

\section*{Acknowledgements}
We would like to first thank the anonymous reviewers for their insightful comments.

Yejia Liu, Jianyi Yang, Pengfei Li, and Shaolei Ren were supported in part by the US NSF under grants CNS-1910208, CNS-2007115, and CCF-2324941.

Tongxin Li was supported in part by the National Natural Science Foundation of China (NSFC) under grant No. 72301234, Pengcheng Peacock Research Fund (Category C), the Guangdong Key Lab of Mathematical Foundations for AI (2023B1212010001), the Shenzhen Key Lab of Crowd Intelligence Empowered Low-Carbon Energy Network, and the start-up funding UDF01002773 of CUHK-Shenzhen.
\section*{Impact Statement}
This paper presents work whose goal is to advance the field of Machine Learning to make public AI mores more equitable when serving multiple agents each having distinct downstream decision processes and objectives. There are many potential societal consequences of our work. Notably,
our work can lead to more uniform decision costs among multiple agents sharing a single pubic model.

\bibliography{ref}
\bibliographystyle{icml2024}

\newpage
\appendix
\onecolumn
\section{Appendix}
In the appendix, we offer additional details to complement the main text. The content is organized as follows:
\begin{itemize}
    \item Section~\ref{s:app-grad}. Detailed calculations to derive the gradient in Eq.~(\ref{eq:grad-nob}) of Section~\ref{s:Solver}.
    \item Section~\ref{app:proofs}. Providing proofs for the theorems and propositions in Section~\ref{s:uniTheory}.
    \item Section~\ref{app:impl}. Additional empirical details of implementation, datasets, and results for case studies in Section~\ref{s:exps}.
    \item Section~\ref{s:app-mixed-a}. Additional experiments where downstream agents have different objective cost functions.
    \item Section~\ref{app:withBeta}. Proposal of a combined objective that explicitly incorporates the loss of the public model, $\mathcal{L}_f$, via a balancing hyperparameter $\beta$, providing a more nuanced control over the tradeoff between fairness and accuracy. We also provide the  empirical results associated with this objective.
\end{itemize}

\subsection{Details of Computing the Gradient}
\label{s:app-grad}
Completing Section~\ref{s:Solver}, we present a detailed derivation of $\nabla_\theta \mathbb{E}[\mathcal{L}^q(\theta)]$. 
By the definition of $\mathcal{L}^q(\theta)$ in the main context, the expected cost is defined as
\begin{align}
\hspace{-1.5em}
\begin{split}
     \mathbb{E}[\mathcal{L}^q(\theta)] &= \mathbb{E}_{(X,Y,\hat{Y},\Xi) \sim p_\theta } \biggl[\sum_{m=1}^{M}  \Bigr(\frac{1}{B_m} \sum_{i=1}^{B_m}  C_{m, i}\Bigr)^{q+1}  \biggl]\\
     &= \int p_\theta (X,Y,\hat{Y},\Xi) \biggl[  \sum_{m=1}^{M} \bigr(\frac{1}{B_m}  \sum_{i=1}^{B_m}  C_{m, i}\bigr)^{q+1} 
\biggl] d (X,Y,\hat{Y},\Xi).
\label{eq:grad-int}
\end{split}
\end{align} 
Subsequently, we obtain
\begin{align*}
\begin{split}
\nabla_\theta \mathbb{E}[\mathcal{L}^q(\theta)]
=&  \int \nabla_\theta p_\theta (X,Y,\hat{Y},\Xi) \biggl[  \sum_{m=1}^{M} \bigr(\frac{1}{B_m} \sum_{i=1}^{B_m}  C_{m, i}\bigr)^{q+1}   \biggl] d (X,Y,\hat{Y},\Xi) \nonumber\\ 
 =& \int p_\theta (X,Y,\hat{Y},\Xi) \nabla_\theta \log p_\theta (X,Y,\hat{Y},\Xi) \biggl[  \sum_{m=1}^{M}  \bigr(\frac{1}{B_m} \sum_{i=1}^{B_m} C_{m, i}\bigr)^{q+1}   \biggl] d (X,Y,\hat{Y},\Xi).
\end{split}
\end{align*}

By decomposing the joint distribution $p_\theta(X,Y,\hat{Y},\Xi)$ based on the chain rule as $$p_\theta(X,Y,\hat{Y},\Xi)= P(\Xi)\cdot P(Y\mid X)\cdot \sigma_{\theta}(\hat{Y}\mid X)\cdot P(X),$$ we have
\begin{align}
\hspace{-1.5em}
\label{eq:grad-mid}
  \nabla_\theta \mathbb{E}[\mathcal{L}^q(\theta)]
=&   \int p_\theta (X,Y,\hat{Y},\Xi) \nabla_\theta \log [P(\Xi)\cdot P(Y\mid X)\cdot \sigma_{\theta}(\hat{Y}\mid X)\cdot P(X)] 
     \biggl[  \sum_{m=1}^{M}  \bigr(\frac{1}{B_m} \sum_{i=1}^{B_m} C_{m, i}\bigr)^{q+1} 
      \biggl]
d (X,Y,\hat{Y},\Xi) \nonumber\\ 
    =& \int p_\theta (X,Y,\hat{Y},\Xi) \nabla_\theta \log \sigma_\theta (\hat Y|X)  \biggl[  \sum_{m=1}^{M}   \bigr(\frac{1}{B_m} \sum_{i=1}^{B_m} C_{m, i}\bigr)^{q+1} 
     \biggl]
 d (X,Y,\hat{Y},\Xi) \nonumber\\
    =& \mathbb{E}_{(X,Y, \hat{Y}, \Xi) \sim p_\theta }\left\{ \nabla_\theta \log \sigma_\theta(\hat Y|X) \biggl[  \sum_{m=1}^{M} \bigr(\frac{1}{B_m} \sum_{i=1}^{B_m} C_{m, i}\bigr)^{q+1} 
     \biggl]\right\}\\
     \nonumber
     =& \mathbb{E}_{(X,Y, \hat{Y}, \Xi) \sim p_\theta }\left\{ \biggl[\sum_{m=1}^{M}\sum_{i=1}^{B_m}\nabla_\theta\log \sigma_\theta (\hat y_{m,i}| x_{m,i})\biggl]\biggl[  \sum_{m=1}^{M} \bigr(\frac{1}{B_m} \sum_{i=1}^{B_m} C_{m, i}\bigr)^{q+1} 
     \biggl]\right\}.
\end{align}
Rewriting Eq.~(\ref{eq:grad-mid}), we obtain the gradient stated in the Eq.~(\ref{eq:grad-nob}).

\subsection{Theoretical Proofs}
\label{app:proofs}
\subsubsection{Proof of Theorem \ref{proposition:beta-0-q-1}}
\begin{proof} Let $\theta^*_{q=0}$ and $\theta^*_{q=1}$ denote optimal solutions of $\min_\theta \mathcal{L}^{q=0}(\theta)$ and $\min_\theta \mathcal{L}^{q=1}(\theta)$ respectively. It follows that
\begin{equation}
\label{eq:lemma1}
    \begin{split}
    \mathsf{Var}&(C_1(\theta^*_{q=1}), ..., C_M(\theta^*_{q=1}))=\frac{1}{M}\sum_{m=1}^{M} C_m^{2}(\theta^*_{q=1}) - \Bigr(\frac{1}{M}\sum_{m=1}^{M} C_m(\theta^*_{q=1})\Bigr)^2 \\
         & \leq \frac{1}{M}\sum_{m=1}^{M} C_m^{2}(\theta^*_{q=0}) -  \Bigr(\frac{1}{M}\sum_{m=1}^{M} C_m(\theta^*_{q=1})\Bigr)^2
         \\
        & \leq \frac{1}{M}\sum_{m=1}^{M} C_m^2(\theta^*_{q=0})- \Bigr(\frac{1}{M}\sum_{m=1}^{M} C_m(\theta^*_{q=0})\Bigr)^2 = \mathsf{Var}(C_1(\theta^*_{q=0}), ..., C_M(\theta^*_{q=0})),
    \end{split}
\end{equation}
where the first inequality holds since $\theta^*_{q=1}$ minimizes $\frac{1}{M}\sum_{m=1}^{M} C_m^{2}(\theta^*_{q=1})$, and the second inequality holds because $\theta^*_{q=0}$ minimizes $\frac{1}{M}\sum_{m=1}^{M} C_m(\theta^*_{q=0})$.
\end{proof}

\subsubsection{Proof of Theorem~\ref{proposition:2}}
To prove Theorem~\ref{proposition:2}, it suffices to show that for any $q \in \mathbb{R}^+$, $M \in \mathbb{N}$, a small increase in $q$ can result in a more equitable solution for the Equitable Objective, based on Definition~\ref{def2}.
Specifically, we prove the derivative of $H_\text{norma}\bigr(C^{q+1}(\theta^*_p)\bigr)$ w.r.t. the variable $p$ at the point $p=q$ is non-negative, i.e.,
\begin{equation}
\label{eq:p=q_proof}
    \frac{\partial \mathbb{H}_\text{norm}\bigr(C^{q+1}(\theta^*_p)\bigr)}{\partial p}|_{p=q} \geq 0.
\end{equation}

\begin{proof}[Proof of the statement above]
For simplicity of notation, we denote the gradient of $C^{q+1}(\theta)$ with respect to $\theta$ as the vector $\nabla_{\theta} C^{q+1}(\theta)$, and the second order derivative  of $C^{q+1}(\theta)$ with respect to $\theta$ as the Hessian matrix $\nabla^2_{\theta} C^{q+1}(\theta)$. If $C(\theta) \neq 0$, we can easily verify that the Hessian matrix $\nabla^2 C^{q+1}(\theta)$ is positive definite for all $q \geq 0$. More specifically, we have
\begin{equation}
    \nabla_{\theta} \left(\nabla_{\theta} C^{q+1}(\theta)\right) = (q+1) \nabla_{\theta} \left( C^{q}(\theta)\nabla_{\theta} C(\theta)\right) = (q+1) C^{q}(\theta)\nabla_{\theta}^2 C(\theta) + (q+1)q C^{q-1}(\theta) \nabla_{\theta} C(\theta) \nabla_{\theta} C(\theta)^\top .
\end{equation}
By definition, $\nabla_{\theta}^2 C(\theta)$ is positive definite and $C(\theta) \nabla_{\theta} C(\theta)^\top$ is semi-positive definite. Since all the coefficients are non-negative, we  conclude the Hessian matrix $\nabla^2_{\theta} C^q(\theta)$ is positive definite when $C(\theta) \neq 0$. 

If $C(\theta) = 0$, both the vector $\nabla_{\theta} C^{q+1}(\theta)$ and the matrix $\nabla^2_{\theta} C^{q+1}(\theta)$ are equal to zero. 

Subsequently, the proof of Eq.~(\ref{eq:p=q_proof}) proceeds as follows:
\begin{equation}\label{eq:app-eq-19}
    \begin{split}
         \frac{\partial \mathbb{H}_\text{norm}\bigr(C^{q+1}(\theta^*_p)\bigr)}{\partial p}|_{p=q}   =& - \frac{\partial}{\partial p} \sum_{m} \frac{C_m^{q+1}(\theta_p^*)}{\sum_m C_m^{q+1}(\theta_p^*)} \ln \Bigr(\frac{C_m^{q+1}(\theta_p^*)}{\sum_m C_m^{q+1}(\theta_p^*)}\Bigr)|_{p=q}\\
         =& - \frac{\partial}{\partial p} \sum_{m} \frac{C_m^{q+1}(\theta_p^*)}{\sum_m C_m^{q+1}(\theta_p^*)} \ln\Bigr( C_m^{q+1}(\theta_p^*)\Bigr)|_{p=q} + \frac{\partial}{\partial p} \ln \sum_m C_m^{q+1}(\theta_p^*)|_{p=q} .
    \end{split}
\end{equation}
For the second term in Eq.~\eqref{eq:app-eq-19}, we have
\begin{equation}
    \begin{aligned}
        \frac{\partial}{\partial p} \ln \sum_m C_m^{q+1}(\theta_p^*)|_{p=q}  &= \frac{\sum_{m} \nabla_{\theta} C_m^{q+1}(\theta_p^*)^\top \cdot \frac{\partial \theta_p^*}{\partial p} }{\sum_m C_m^{q+1}(\theta_p^*)}|_{p=q}\\
        &=  \frac{1}{\sum_m C_m^{q+1} (\theta_p^*)} \cdot \frac{\partial \theta_p^*}{\partial p}^\top|_{p=q} \cdot \sum_{m} \nabla_{\theta} C_m^{q+1}(\theta_p^*).
    \end{aligned}
\end{equation}
Since the $\theta_p^*$ is an optimal solution for the $\mathcal{L}^p(\theta)$ objective, then for $q=p$, by definition we have $\sum_{m} \nabla_{\theta}  C_m^{q+1}(\theta_p^*)= 0$. Therefore, the second term of Eq.~\eqref{eq:15} is zero. The derivative can then be rewritten as 
\begin{equation}
\label{eq:15}
    \begin{split}
         \frac{\partial \mathbb{H}_\text{norm}\bigr(C^{q+1}(\theta^*_p)\bigr)}{\partial p}|_{p=q} = & -\sum_m \frac{ (\frac{\partial}{\partial p} \theta_p^*|_{p=q})^\top \nabla_\theta C_m^{q+1}(\theta_p^*)}{\sum_m C_m^{q+1}(\theta_p^*)} \ln(C_m^{q+1}(\theta_p^*)) \\
         & -  \sum_m \frac{C_m^{q+1}(\theta_p^*)}{\sum_m C_m^{q+1}(\theta_p^*)} \frac{ (\frac{\partial}{\partial p} \theta_p^*|_{p=q-1}) ^T\nabla_\theta C_m^{q+1}(\theta_p^*)}{ C_m^{q+1}(\theta_p^*)} \\
         =& -\sum_m \frac{ (\frac{\partial}{\partial p} \theta_p^*|_{p=q})^\top\nabla_\theta C_m^{q+1}(\theta_p^*)}{\sum_m C_m^{q+1}(\theta_p^*)} (\ln(C_m^{q+1}(\theta_p^*))+1). 
    \end{split}
\end{equation}
Here, if for all $M \in \mathbb{N}$, the costs $C_m(\theta_p^*)$ are all zero costs. Therefore, we see that $\frac{\partial \mathbb{H}_\text{norm}\bigr(C^{q+1}(\theta^*_p)\bigr)}{\partial p}|_{p=q} = 0 $, leading to the desirable result.

For the non-trivial case, there exists some $M\in \mathbb{N}$ such that $C_m(\theta_p^*)>0$. Since $\theta_p^*$ is an optimal solution of our objective function, we have $\sum_{m} \nabla_{\theta}  C_m^{p+1}(\theta_p^*)= 0$ for all $p\geq 0$. In other words, $\frac{\partial}{\partial p} \sum_{m} \nabla_{\theta}C_m^{p+1}(\theta_p^*)= 0$. Then we can calculate the gradient as follows
\begin{equation}
    \begin{aligned}
        &\frac{\partial}{\partial p} \sum_{m} \nabla_{\theta}C_m^{p+1}(\theta_p^*) \\
        = & \sum_m \nabla_\theta^2 C_m^{p+1}(\theta_p^*) \frac{\partial} {\partial_p} \theta_p^* + \sum_m \Bigl( C_m^p(\theta_p^*) + (p+1) C_m^p(\theta_p^*) \ln(C_m(\theta_p^*)) \Bigr) \nabla_{\theta}  C_m(\theta_p^*)\\
        = & \sum_m \nabla_\theta^2 C_m^{p+1}(\theta_p^*) \frac{\partial} {\partial_p} \theta_p^* + \frac{1}{p+1}\sum_m \Bigl( (p+1) C_m^p(\theta_p^*) \nabla_{\theta}C_m(\theta_p^*)   \Bigr) + \Bigl( (p+1) C_m^p(\theta_p^*) \nabla_{\theta}C_m(\theta_p^*)  \ln(C_m(\theta_p^*))  \Bigr) \\
        =& \sum_m \nabla_\theta^2 C_m^{p+1}(\theta_p^*) \frac{\partial}{\partial_p} \theta_p^* + \frac{1}{p+1}\sum_m \Bigl(\ln (C_m^{p+1}(\theta_p^*)) + 1)\nabla_\theta C_m^{p+1}(\theta_p^*) \Bigr).
    \end{aligned}
\end{equation}
To summarize, we have
\begin{equation}
\label{eq:16}
    \sum_m \nabla_\theta^2 C_m^{p+1}(\theta_p^*) \frac{\partial}{\partial_p} \theta_p^* + \frac{1}{p+1}\sum_m (\ln (C_m^{p+1}(\theta_p^*)) + 1)\nabla_\theta C_m^{p+1}(\theta_p^*) = 0.
\end{equation}

In our non-trivial case, there exists at least one $m\in \mathbb{N}$, such that the Hessian matrix $\nabla_\theta^2 C_m^{p+1}(\theta_p^*)$ is positive definite. Then the matrix $\biggl({\sum_m \nabla_{\theta}^2 C_m^{p+1}(\theta_p^*)}\biggl)$ is also positive definite. 
Therefore, we can calculate the gradient $\frac{\partial}{\partial_p} \theta_p^*$ as below 
\begin{equation}
\label{eq:17}
    \frac{\partial}{\partial_p} \theta_p^* = -\frac{1}{p+1} \biggl({\sum_m \nabla_\theta^2 C_m^{p+1} (\theta_{p}^*)}\biggl)^{-1} \sum_m (\ln (C_m^{p+1}(\theta_p^*)) +1)\nabla_\theta C_m^{p+1}(\theta_p^*).
\end{equation}

Plugging Eq.~(\ref{eq:17}) into Eq.~(\ref{eq:15}), we have
\begin{equation}
    \begin{aligned}
        \frac{\partial \mathbb{H}_\text{norm}\bigr(C^{q+1}(\theta^*_p)\bigr)}{\partial p} \biggl|_{p=q} = & \frac{\sum_m (\ln (C_m^{q+1}(\theta_p^*) )+1)\nabla_\theta C_m^{q+1}(\theta_p^*)^\top}{(p+1) \sum_m C_m^{q+1}(\theta_p^*)} \biggl({\sum_m \nabla_\theta^2 C_m^{p+1}(\theta_q^*)}\biggl)^{-1}\\
        &\cdot \sum_m \nabla_\theta C_m^{p+1}(\theta_p^*)(\ln(C_m^{p+1}(\theta_p^*))+1) \biggl|_{p=q}.
    \end{aligned}
\end{equation}
Since the matrix $\biggl({\sum_m \nabla_{\theta}^2 C_m^q(\theta_p^*)}\biggl)$ is positive definite and the coefficient $q\sum_m C_m^q(\theta_p^*)$ is positive, we conclude that 
$\frac{\partial \mathbb{H}_\text{norm}\bigr(C^{q+1}(\theta^*_p)\bigr)}{\partial p}|_{p=q} \geq 0$.
\end{proof}

As a result,  Eq.~(\ref{eq:p=q_proof}) implies that for any $p$, the performance distribution of $\{C_1^p(\theta_{p+\epsilon}^*), ..., C_M^p(\theta_{p+\epsilon}^*)\}$ exhibits greater uniformity compared to the distribution of $\{C_1^p(\theta_{p}^*), ..., C_M^p(\theta_{p}^*)\}$, provided that the value of $\epsilon$ is sufficiently small. 

\begin{corollary}
    Let $C(\theta)$ be twice differentiable in $\theta$ with $\nabla^2 C(\theta) > 0$ (positive definite), for the special case $M = 2$, the derivative of $\mathbb{H}_\text{norm}\bigr(C^{q+1}(\theta^*_p)\bigr)$ w.r.t. the evaluation point $p$ is non-negative for all $p\geq 0$ and $q \geq 0$, \textit{i.e.},
\begin{equation}
    \frac{\partial \mathbb{H}_\text{norm}\bigr(C^{q+1}(\theta^*_p)\bigr)}{\partial p} \geq 0.
\end{equation}
\end{corollary}

\begin{proof}
Let $w_q(\theta) = \frac{C_1^{q+1}(\theta)}{C_1^{q+1}(\theta) + C_2^{q+1}(\theta)}$. Without loss of generality, we assume $w_q(\theta_p^*) \in (0, \frac{1}{2})$. 
If $w_q(\theta_p^*) = \frac{1}{2}$, the gradient of norm $\mathbb{H}_\text{norm}$ is defined as $-\ln(\frac{w_q(\theta_p^*)}{1 - w_q(\theta_p^*)}) \cdot \frac{\partial w_q(\theta_p^*)}{\partial p}$, which trivially equals to zero for any $q$ and $p$. If $w_q(\theta) \in (\frac{1}{2}, 1)$, we can flip the label of $C_1$ and $C_2$ to make sure $w_q(\theta) \in (0, \frac{1}{2})$

Given $M=2$, by applying the chain rule, the gradient of the norm can be rewritten as 
\begin{equation}
    \begin{aligned}
        &\frac{\partial \mathbb{H}_\text{norm}\bigr(C^{q+1}(\theta^*_p)\bigr)}{\partial p}\\
        =& -\ln(\frac{w_q(\theta_p^*)}{1 - w_q(\theta_p^*)}) \cdot \frac{\partial w_q(\theta_p^*)}{\partial p} \\
        =& -\ln(\frac{w_q(\theta_p^*)}{1 - w_q(\theta_p^*)}) \cdot \frac{\partial w_q(\theta_p^*)}{\partial \left(\frac{C_1(\theta_p^*)}{C_2(\theta_p^*)}\right)^{q+1}}  \cdot \frac{\partial}{\partial p} \left(\frac{C_1(\theta_p^*)}{C_2(\theta_p^*)}\right)^{q+1}\\
        =& -\ln(\frac{w_q(\theta_p^*)}{1 - w_q(\theta_p^*)}) \cdot \left(\frac{C_2^{q+1}(\theta^*_p)}{C_1^{q+1}(\theta^*_p) + C_2^{q+1}(\theta^*_p)}\right)^2 \frac{\partial}{\partial p} \left(\frac{C_1(\theta_p^*)}{C_2(\theta_p^*)}\right)^{q+1}\\
        =& -\ln(\frac{w_q(\theta_p^*)}{1 - w_q(\theta_p^*)}) \cdot \left(\frac{C_2^{q+1}(\theta^*_p)}{C_1^{q+1}(\theta^*_p) + C_2^{q+1}(\theta^*_p)}\right)^2 
        (q+1) \left(\frac{C_1(\theta_p^*)}{C_2(\theta_p^*)}\right)^{q} \cdot \frac{\partial}{\partial p} \left(\frac{C_1(\theta_p^*)}{C_2(\theta_p^*)}\right) \\
        =& -\ln(\frac{w_q(\theta_p^*)}{1 - w_q(\theta_p^*)}) \cdot \left(1 - w_q(\theta_p^*) \right)^2 
        (q+1) \left(\frac{C_1(\theta_p^*)}{C_2(\theta_p^*)}\right)^{q} \cdot \frac{\partial}{\partial p} \left(\frac{C_1(\theta_p^*)}{C_2(\theta_p^*)}\right).
    \end{aligned}
\end{equation}
For any $q\geq 0$, it's obvious that 
\begin{equation}
    \ln(\frac{1 - w_q(\theta_p^*)}{w_q(\theta_p^*)}) \cdot \left(1 - w_q(\theta_p^*) \right)^2 
        (q+1) \left(\frac{C_1(\theta_p^*)}{C_2(\theta_p^*)}\right)^{q} \cdot \frac{\partial}{\partial p}  \geq 0.
\end{equation}
According to Eq.~\eqref{eq:p=q_proof}, in the point $q = p$, we have $\frac{\partial \mathbb{H}_\text{norm}\bigr(C^{q+1}(\theta^*_p)\bigr)}{\partial p} \geq 0$, which is equivalent to 
\begin{equation}\label{eqn:gradient_of_C}
    \frac{\partial}{\partial p} \left(\frac{C_1(\theta_p^*)}{C_2(\theta_p^*)}\right) \geq 0.
\end{equation}

Since we assume $w_q(\theta) \in (0,\frac{1}{2})$, then $C_1(\theta) < C_2(\theta)$. For any $q' \geq 0$, we also have $w_{q'}(\theta) \in (0,\frac{1}{2})$ and the following
\begin{equation}\label{eqn:q_prime_factor}
    \ln\left(\frac{1 - w_{q'}(\theta_p^*)}{w_{q'}(\theta_p^*)}\right) \cdot \left(1 - w_{q'}(\theta_p^*) \right)^2 
        \cdot (q'+1) \left(\frac{C_1(\theta_p^*)}{C_2(\theta_p^*)}\right)^{q'} \geq 0.
\end{equation}
By multiplying Eq.~\eqref{eqn:gradient_of_C} with Eq.~\eqref{eqn:q_prime_factor}, for $M=2$, we conclude for any $p \geq 0$ and $q \geq 0$, 
\begin{equation}
    \frac{\partial \mathbb{H}_\text{norm}\bigr(C^q(\theta^*_p)\bigr)}{\partial p} \geq 0.
\end{equation}
\end{proof}

\subsubsection{Proof of Proposition~\ref{proposition:5}}
\begin{proof}
We start with \textit{a specific} $\kappa$. 
\label{proofOfLemma4}
Similar to the proof in~\citet{mohri2019agnostic}, for any $\delta > 0$, the following inequality holds with probability at least $1-\delta$ for  $h \in H$:
\begin{equation}
\label{eq:ineq}
    \mathcal{J}_\kappa(h) \leq \mathcal{L}_\kappa(h) + \mathbb{E}\left[\max_{h \in H} \mathcal{J}_\kappa(h) - \mathcal{L}_\kappa(h)\right] + B\sqrt{\sum_{m} \frac{\kappa_m^2}{2N_m}\log\frac{1}{\delta}}.
\end{equation}
Using the H$\ddot{o}$lder's inequity, we have
\begin{equation}
    \mathcal{L}_\kappa(h) = \sum_m \kappa_m C_m \leq \left(\sum_m \kappa_m^p\right)^{\frac{1}{p}} \left(\sum_{m} C_m^{q+1}\right)^{\frac{1}{q+1}} = ||\kappa||_p\Tilde{\mathcal{L}}^q(h), \frac{1}{p} + \frac{1}{q+1}=1.
\end{equation}
Plugging $\mathcal{L}_\kappa(h) \leq ||\kappa||_p\Tilde{L}^q(h)$ into Eq.~(\ref{eq:ineq}), we obtain for $h \in H$,
\begin{equation}
\label{eq:gen1}
    \mathcal{J}_\kappa(h) \leq ||\kappa||_p\Tilde{\mathcal{L}}^q(h) + \mathbb{E}\left[\max_{h \in H} \mathcal{J}_\kappa(h) - \mathcal{L}_\kappa(h)\right] + B\sqrt{\sum_{m} \frac{\kappa_m^2}{2N_m}\log\frac{1}{\delta}},
\end{equation}
where $\frac{1}{p} + \frac{1}{q+1}=1$.

Therefore, Eq.~(\ref{eq:gen2}) in Proposition \ref{proposition:5}  can be readily derived from Eq.~(\ref{eq:gen1}) by considering the maximum value across all potential $\kappa$ values within $\Delta$.
\end{proof}

\textbf{Discussions} Deriving the optimal value of $q$ that results in the tightest generalization bound from Proposition~\ref{proposition:5} is not trivial. In practice, our proposed Equitable Objective allows us to fine-tune a range of $q$ values to strike a balance between performance equity/uniformity and accuracy.

\subsection{Additional Experiments Details and Results}
\label{app:impl}
\subsubsection{Additional Empirical Details}
\label{app:impl-d}
For the data centers application in Section~\ref{sAPPI}, within each agent, the dataset is randomly partitioned, with $67\%$ allocated as the training set and the remaining portion as the testing set. As for the EV charging application in Section~\ref{sAPPII}, the ratio between training and testing in each agent is 70\% vs. 30\%. We set the learning rate as $0.05$ for the data centers application and $1e-4$ for the EV charging application. We employ the Adam optimizer with a scheduler featuring a step size of $50$ and a decay factor of $0.5$. In both applications, the batch size is set as $128$. For predicting the next time step in the data centers application, a sequence length of $12$ is utilized, while in the EV charging application, the prediction involves the next charging time window spanning $12$ time steps, a sequence length of $12$ is also employed. The LSTM model employed in data centers application has a hidden size of $50$. In the EV charging application, the Transformer model consists of a single-layer encoder-decoder with positional encoding, utilizing a feature size of $250$.

\begin{figure}[h]
    \centering
    \includegraphics[width=0.8\linewidth]{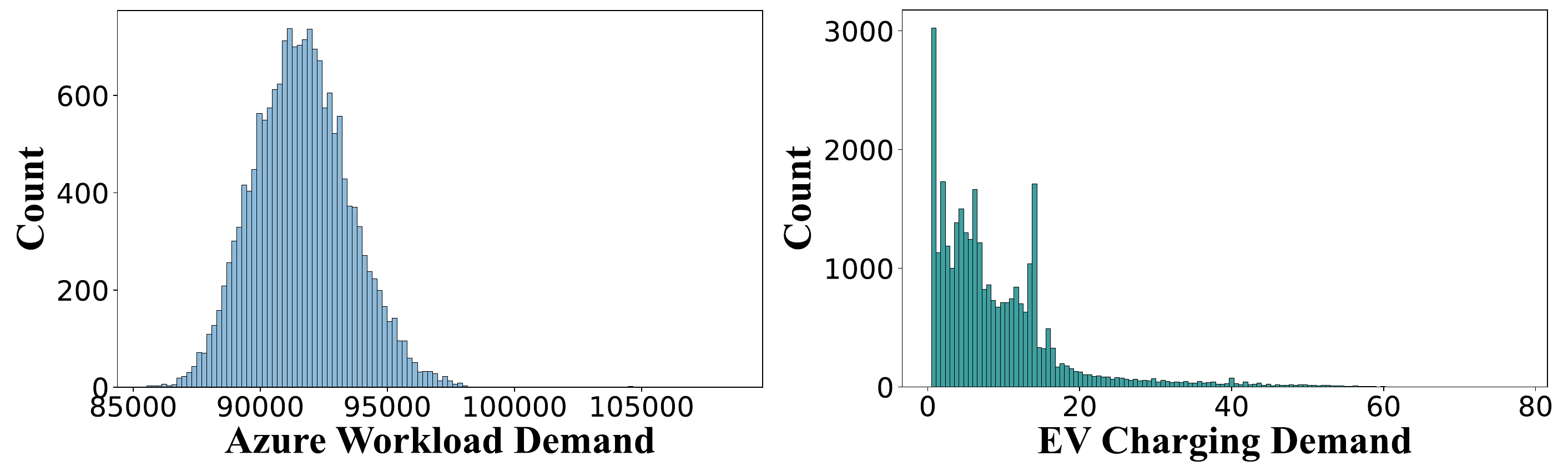}
    \caption{Depictions of (a) Azure workload demands~\cite{shahrad2020serverless}; (b) EV charging demands in ACN-Data~\cite{acndata}.  }
    \label{fig:demands}
\end{figure}

\begin{figure}[h]
    \centering
    \includegraphics[width=\linewidth]{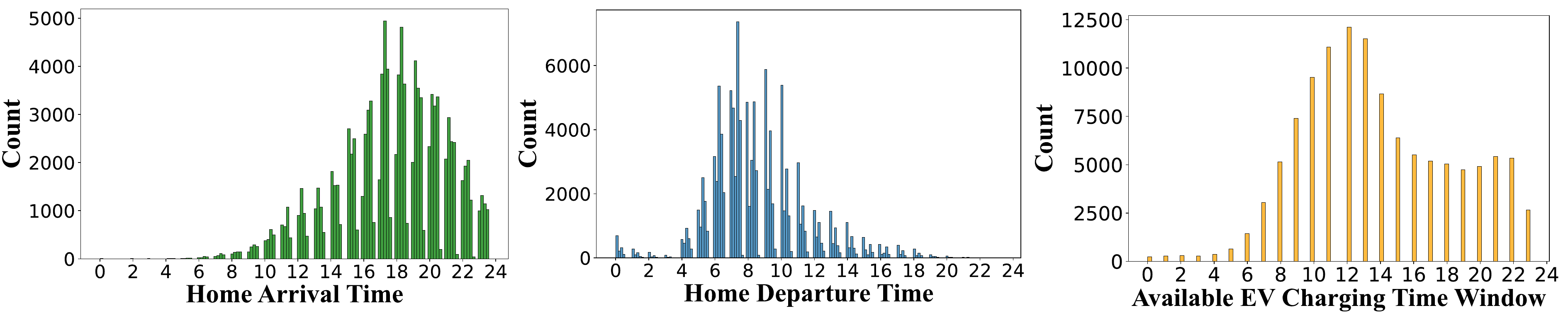}
    \caption{Depictions of home arrival, home departure and available charging time window for residential EV based on the NHTS government data~\cite{NHTS}. }
    \label{fig:home-trips}
\end{figure}

In the EV charging scheduling application of Section~\ref{sAPPII}, we use the publicly available National Household Travel Survey (NHTS) data~\cite{NHTS} to approximate the available charging time window, \textit{i.e.,} from $s_j$ to $e_j$, for residential sectors~\cite{sche}. The NHTS contains the travel logs of $117,222$ American households' vehicles, detailing the number of trips for each household and the start and end times for each trip per day. We assume the distribution for the initial charging time $s_j$ and end time $e_j$ of EV are the same as the distribution of home arrival and home departure times, respectively. We use the time when the last trip of a household concludes from NHTS as the daily home arrival time. Similarly, we designate the time when the first daily trip begins from NHTS as the daily home departure time~\cite{sche}. 

\subsubsection{Additional Details of Datasets}
\label{s:app-data}
We depict the distribution of Azure workload demand~\cite{shahrad2020serverless} and EV charging electricity demands~\cite{acndata} in Figure~\ref{fig:demands}. Additionally, in Figure~\ref{fig:home-trips}, we present the distributions of home arrival time, home departure time, and the available EV charging time window, calculated as the difference between home departure time and home arrival time, utilizing data from the NHTS government dataset~\cite{NHTS}. From Figure~\ref{fig:home-trips}, it is evident that a significant portion of residential households has an available charging time window exceeding 8 hours, thereby supporting the feasibility of scheduling environmentally friendly and financially efficient charging.

In data preprocessing of the EV charging application, we focus on the state of California to ensure alignment between the ACN-Data and CAISO. Besides null value, we also filter out the data points containing charging duration exceeding one day, as most EV can complete full charging within 5 hours, as reported by the government survey~\cite{osti}.

\subsubsection{Additional Results}
\label{s:app-res}
\begin{table}[h] 
 \caption{MSE loss of the Plain PM and the \ouralg with varied $q+1$.}
  \centering
  \resizebox{0.5\textwidth}{!}{
    \setlength{\tabcolsep}{12pt}
  \begin{tabular}{cccc}
    \toprule
    & & \multicolumn{2}{c}{\textbf{MSE loss}} \\
    \cline{3-4}
     & $q+1$ & {Similar Agents} & {Different Agents} \\
    \hline
    \multirow{3}{*}{\ouralg} 
    & 20 & 6.63 & 6.52 \\
     & 30 & 6.57 &  6.47\\
      & 40 & 6.55 &  6.48\\
      \midrule
      Plain PM & - & 6.54 & 6.45\\
    \bottomrule
  \end{tabular}
  }
  \label{tab:app2-nob}
\end{table}
In completing the results of the scheduled EV charging application in Section~\ref{sAPPII}, we report the MSE loss of each method under conditions where the distributions w.r.t. $(D_j - I_j)$ of downstream agents are ``similar'' and ``different'' in Table~\ref{tab:app2-nob}.

\subsection{Experiments on Diverse Cost Objectives of Downstream Agents}
\label{s:app-mixed-a}
We add an experiment where agents have different objective functions: \textit{Agent (A)} for data center workload scheduling, \textit{Agent (B)} for EV charging, and \textit{Agent (C)} for iPhone green charging. This setup creates a diverse pool of agents with varying objectives, all utilizing carbon emission predictions from the upstream public model. The objectives for Agent (A) and (B) are defined by Eq.~(\ref{eq:carbon}) and Eq.~(\ref{eq:cwp}) in the main text, respectively. Note that for the EV charging application in this experiment, the public model only predicts carbon emissions $E_t^C$ rather than $E_t$. For iPhone green charging, the objective is to minimize carbon emissions by optimizing the charging schedule, formulated as $\min_{X_{o}} \sum_t \mu_{o, t} x_{o,t} \cdot E^C_{t}$, where $\mu_{o,t}$ represents the electricity charged for the $o$-th iPhone at time $t$, $x_{o, t}$ is a binary variable ($x_{o, t} \in \{0, 1\}$) indicating whether charging occurs at time $t$, and $X_o=[x_{o,1},\cdots, x_{o,T}]$, denoting the charging schedule for the $o$-th iPhone.

In the implementation, we set $\lambda$ to $2$ for the objective of Agent (A) as indicated by Eq.~(\ref{eq:carbon}). The dataset is split into training and testing sets with a ratio of $67\%$ to $33\%$. We set the initial learning rate to $0.05$ for training the Plain Public Models, and $0.1$ for training \ouralg, with a step size of 50 and a decay rate of $0.1$. The batch size is set to $128$ for training both models. In this experiment, we use the transformer with the same architecture described in Section~\ref{sAPPII}. For the three downstream agents with diverse objectives, we set the sequence length to $12$ when predicting the next time steps of carbon emissions. In the cases of EV charging and iPhone green charging, the length of the available time frame is set to $12$. For the data center application in Agent (A), which only requires the immediate next time step of carbon emission prediction, we average the predicted next 12 time steps of carbon emissions from the upstream public model. For Agent (B) and Agent (C), which need predictions for the next 12 time steps of carbon emissions, we use the predicted values directly.
 
We present the results of using different objectives across downstream agents in Table~\ref{tab:mixed-agents}. The results indicate that even when downstream agents have distinct objective functions, our proposed \ouralg still reduces the variance in their performance distribution. This leads to a fairer solution compared to the Plain PM, which only minimizes carbon prediction error without considering the decision-making costs of diverse downstream agents.

\begin{table*}[h]
    \setlength\intextsep{10pt}
    \setlength\lineskip{10pt}
    \captionsetup{width=0.9\linewidth}
     \caption{Statistics of the test results using different cost objectives for downstream agents. 
     }
    \centering
    \resizebox{0.8\textwidth}{!}{
    \setlength{\tabcolsep}{16pt}
    \begin{tabular}{cccccc}
    \toprule
 
   Method &  $q+1$ &  \textbf{Variance} & \textbf{ Mean} & $\bm{C_{95} - C_{5}}$ & MSE \\
    \toprule
     \multirow{2}{*}{\textbf{\ouralg}} & 
     $1$ & 18.14 & 7.06 & 9.28  & 9.66\\
     & $1.1$ & \bf 15.72 & \bf 6.14 & \bf 8.39 & 9.67\\
   
    \midrule
    Plain PM & -  & 18.89 & 7.24 & 9.45 & \bf 7.20\\
    
     \bottomrule
    \end{tabular}
    }
    \label{tab:mixed-agents}

\end{table*}

\subsection{Combined Objective: Explicitly Incorporating $\mathcal{L}_f$}
\label{app:withBeta}
We present a combined objective here to complement the Equitable Objective proposed in the main text to provide a more nuanced control over equity/fairness versus model accuracy. The combined objective shown in Eq.~(\ref{eq:obj}) incorporates the loss of public model $\mathcal{J}_f$ into the original Equitable Objective,
\begin{align}
\begin{split}
     &\min_\theta \quad (1-\beta) \mathcal{J}_{EQ}^q  + \beta \mathcal{J}_{f}, \quad \text{with}\\
     \mathcal{J}_{EQ}^q &= \sum_{m=1}^M \mathbb{E}^{q+1} \left[\text{cost}_m (\hat a_m, \xi_m, y) - \text{cost}_m ( a_m, \xi_m, y)\right] \\
     \mathcal{J}_{f}  & =\mathbb{E}[\|y-\hat{y}\|^2],
\label{eq:obj}
\end{split}
\end{align}
where $\beta$ controls the weighting of each component. We then approximate the expectation in Eq.~(\ref{eq:obj}) with the empirical loss as shown in the Eq.~(\ref{eq:obj-mid}).
\begin{align}\label{eq:obj-mid}
\begin{split}
   & \min_\theta \quad  (1-\beta) \mathcal{L}_{EQ}^q  + \beta \mathcal{L}_{f}, \quad \text{with} \\
    \mathcal{L}_{EQ}^q = &\sum_{m=1}^M \biggl\{\Bigr[\frac{1}{N_m} \sum_{i=1}^{N_m} \Bigr(\text{cost}_m (\hat a_{m,i}, \xi_{m, i}, y_{m,i}) -  \text{cost}_m ( a_{m,i}, \xi_{m, i}, y_{m,i})\Bigr)\Bigr]^{q+1}\biggl\}
      \\
      \mathcal{L}_{f} = &\sum_{m=1}^M \frac{1}{N_m} \sum_{i=1}^{N_m}\|y_{m,i}-\hat{y}_{m,i}\|^2
\end{split}
\end{align}
Likewise, if the cost functions are differentiable, the gradient of the combined objective is calculated as
\begin{align*}
\begin{split}
 \nabla_\theta ((1-\beta) \mathcal{L}_{EQ}^q  + \beta \mathcal{L}_{f}) =&  (1-\beta) \sum_{m=1}^M \sum_{i=1}^{N_m} \nabla_{C_{m, i}} \mathcal{L}_{EQ}^q \nabla_{\hat a_{m, i}} C_{m, i} \nabla_{\hat y_i} \hat a_{m, i} \nabla_{\theta} \hat{y}_{m,i}  \\
&+ \beta \sum_{m=1}^M \sum_{i=1}^{N_m} \frac{2}{N_m}(\hat{y}_{m,i}-y_{m,i})\nabla_\theta \hat{y}_{m,i}, 
\end{split}
\end{align*}

where $C_{m, i} =  \text{cost}_m (\hat a_{m,i}, \xi_{m, i}, y_{m,i}) -\text{cost}_m ( a_{m,i}, \xi_{m, i}, y_{m,i}).$

If the cost functions are non-differentiable, similar as \eqref{eq:grad-nob-empirical}, given a training dataset with $K$ batches and a batch size $B_m$, the gradient can be calculated as
\begin{align}
\begin{split}
     \nabla_\theta \mathcal{L}^q(\theta) \!= \!  \frac{1}{K}\sum_{k=1}^K\!\Bigr\{\biggl[\sum_{i=1}^{B_m}\!\! \sum_{m=1}^{M} \!\!\nabla_\theta\!\log \sigma_\theta (\hat y_{m,k,i}| x_{m,k,i})\biggl] 
   \cdot\biggl[   \sum_{m=1}^{M}  \Big[(1-\beta)\Bigr(\frac{1}{B_m} \sum_{i=1}^{B_m} C_{m, k,i}\Bigr)^{q+1}+\beta \sum_{i=1}^{B_m} \frac{1}{B_m}  \mathcal{L}_{f,m, i} \Big]
  \biggl] \Bigr\}.
\end{split}
\end{align}

It is not straightforward to prove that a larger $q$ would lead to a more uniform cost regret distribution by the combined objective. The challenge arises because $\theta'$ that minimizes $\mathcal{L}_{EQ}^{q}$ may not align with $\theta^*$ that optimizes the combined loss of $\mathcal{L}_{EQ}^q$ and $\mathcal{L}_f$. Nevertheless, we highlight that the combined objective provides a way to allow us to balance between fairness of downstream agents and upstream public model accuracy, achieved by adjusting the value of $\beta$. 

\begin{figure*}[t]
    \centering
\includegraphics[width=.9\textwidth]{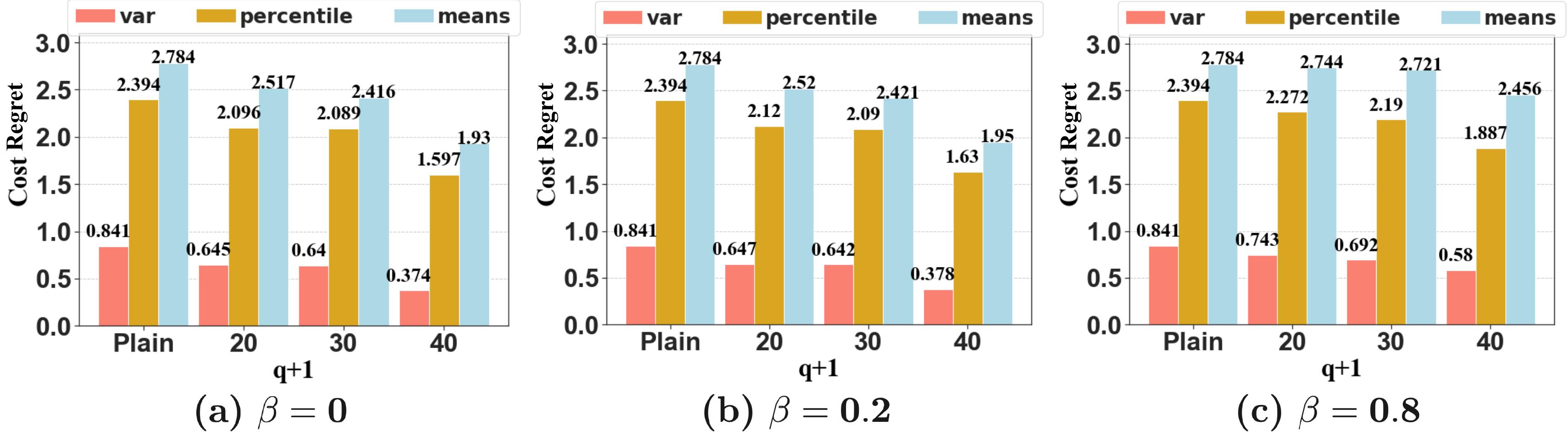}
   \caption{Statistics of test results when $\beta=[0, 0.2, 0.8]$.  Note the \ouralg here refers to the combined objective, Eq.~\eqref{eq:obj}. We can observe that the \ouralg has achieved more uniform distributions among agents compared to the Plain PM, according to the variance and percentile difference measures. }
    \label{fig:app2}
\end{figure*}

\begin{table}[h]
 \caption{MSE loss of the Plain PM and the \ouralg with $q+1$ as $40$. Note the \ouralg here refers to the combined objective, Eq.~(\ref{eq:obj}). As $\beta$ increases, the MSE loss of \ouralg decreases.}
  \centering
   \resizebox{0.35\textwidth}{!}{
  \begin{tabular}{ccc}
    \toprule
     & & \textbf{MSE loss} \\
    \hline
    \multirow{3}{*}{\textbf{\ouralg}} 
    & $\beta=0$  & 6.48\\
     & $\beta=0.2$ & 6.48 \\
      & $\beta=0.8$ & 6.46  \\
      \midrule
      Plain & - & 6.45\\
    \bottomrule
  \end{tabular}
  }
  \label{tab:app2}
\end{table}

\subsubsection{Empirical Results for the Combined Objective}
We perform empirical investigations under the same setup outlined in Section~\ref{sAPPII} to examine whether the proposed combined objective in Eq.~(\ref{eq:obj-mid}) could lead to a more equitable performance distribution among agents in the EV Charging Scheduling case study. Various $\beta$ and $q$ values are considered. Note the \ouralg mentioned in the following results refers to the public model trained using the combined objective in Eq.~(\ref{eq:obj-mid}).

\paragraph{Results}  Figure~\ref{fig:app2} reports the evaluation results between the Plain PM and \ouralg with different $q$ and $\beta$ values. It can be observed the variance and $C_{95}-C_{5}$ achieved by the \ouralg consistently remain lower than using the Plain PM. From Figure~\ref{fig:app2}, we observe both the variance and $C_{95}-C_{5}$ of cost regret distributions across agents decreases as the value of $q$ increases, implying the performance distribution becomes more uniform. Notably, setting $\beta=0$ makes the \ouralg focus on optimizing the Equitable Objective $\mathcal{L}^q_{EQ}$ exclusively, resulting in the most uniform distribution compared to $\beta=0.2$ and $\beta=0.8$. In contrast, the MSE loss, $\mathcal{L}_f$, decreases as $\beta$ increases, as shown in Table~\ref{tab:app2}.

\end{document}